\documentclass[11pt,letterpaper]{article}

\usepackage[margin=1in]{geometry}
\usepackage{amsmath,amssymb,amsthm,mathtools}
\usepackage{enumitem}
\usepackage{booktabs}
\usepackage{array}
\usepackage{tabularx}
\usepackage{aliascnt}
\usepackage{microtype}
\usepackage[numbers,sort&compress]{natbib}
\usepackage[colorlinks=true,linkcolor=blue,citecolor=blue,urlcolor=blue]{hyperref}
\usepackage[nameinlink,noabbrev]{cleveref}
\newtheorem{theorem}{Theorem}[section]

\newaliascnt{proposition}{theorem}
\newtheorem{proposition}[proposition]{Proposition}
\aliascntresetthe{proposition}

\newaliascnt{lemma}{theorem}
\newtheorem{lemma}[lemma]{Lemma}
\aliascntresetthe{lemma}

\newaliascnt{corollary}{theorem}

\aliascntresetthe{corollary}

\newaliascnt{assumption}{theorem}
\newtheorem{assumption}[assumption]{Assumption}
\aliascntresetthe{assumption}

\theoremstyle{definition}
\newaliascnt{definition}{theorem}

\aliascntresetthe{definition}

\theoremstyle{remark}
\newaliascnt{remark}{theorem}
\newtheorem{remark}[remark]{Remark}
\aliascntresetthe{remark}

\numberwithin{equation}{section}

\crefname{theorem}{Theorem}{Theorems}
\Crefname{theorem}{Theorem}{Theorems}
\crefname{lemma}{Lemma}{Lemmas}
\Crefname{lemma}{Lemma}{Lemmas}
\crefname{assumption}{Assumption}{Assumptions}
\Crefname{assumption}{Assumption}{Assumptions}
\crefname{definition}{Definition}{Definitions}
\Crefname{definition}{Definition}{Definitions}
\crefname{remark}{Remark}{Remarks}
\Crefname{remark}{Remark}{Remarks}
\crefname{proposition}{Proposition}{Propositions}
\Crefname{proposition}{Proposition}{Propositions}
\crefname{corollary}{Corollary}{Corollaries}
\Crefname{corollary}{Corollary}{Corollaries}
\crefname{appendix}{Appendix}{Appendices}
\Crefname{appendix}{Appendix}{Appendices}
\crefname{section}{Section}{Sections}
\Crefname{section}{Section}{Sections}
\crefname{subsection}{Section}{Sections}
\Crefname{subsection}{Section}{Sections}

\newcommand{\R}{\mathbb R}
\newcommand{\E}{\mathbb E}
\newcommand{\dd}{\,\mathrm d}
\newcommand{\dx}{\,\mathrm dx}
\newcommand{\norm}[1]{\left\lVert#1\right\rVert}
\newcommand{\ip}[2]{\left\langle#1,#2\right\rangle}
\DeclareMathOperator{\tr}{tr}
\DeclareMathOperator{\Div}{div}
\DeclareMathOperator{\prox}{prox}
\DeclareMathOperator{\Ent}{Ent}
\newcommand{\heat}{\mathsf H}
\newcommand{\gen}{\mathcal L_\lambda}
\newcommand{\inv}{\widehat\pi_{\lambda,h}}

\newcommand{\drift}{\nabla U_\lambda}
\newcommand{\hess}{\nabla^2U_\lambda}
\newcommand{\Lp}[2]{\left\lVert#1\right\rVert_{L^{#2}(\pi_\lambda)}}
\newcommand{\Nminus}[1]{\left\lVert#1\right\rVert_{\dot H^{-1}(\pi_\lambda)}}
\allowdisplaybreaks[1]
\hypersetup{pdftitle={Near-Linear Accuracy Bounds for Moreau--Yosida Unadjusted Langevin Sampling},pdfauthor={},linktoc=page}

\title{Near-Linear Accuracy Bounds for\\
Moreau--Yosida Unadjusted Langevin Sampling}
\author{Yuchen Xin\textsuperscript{*},
Zhihua Zhang\textsuperscript{$\dagger$}}
\date{}

\begin{document}
\maketitle

\begingroup
\renewcommand{\thefootnote}{\fnsymbol{footnote}}
\footnotetext[1]{School of Mathematical Sciences, Peking University; email: \texttt{2301110087@pku.edu.cn}}
\footnotetext[2]{School of Mathematical Sciences, Peking University; email: \texttt{zhzhang@math.pku.edu.cn}}
\endgroup

\begin{abstract}
We establish near-linear accuracy bounds for the classical Moreau--Yosida unadjusted Langevin algorithm (MYULA).
The target is $\pi\propto e^{-f-g}$, where $f\in C^2(\R^d)$ is $m$-strongly convex with Lipschitz gradient and $g$ is convex and globally Lipschitz. Under an explicit parameter-dependent step-size condition, we bound the invariant-measure bias relative to the Moreau-smoothed target by $\widetilde O(h)$, with only logarithmic dependence on the inverse smoothing parameter in the error coefficient. Combining this estimate with the Moreau approximation bias and Wasserstein contraction gives $\widetilde O(\varepsilon^{-1})$ iterations to make the $N$th-iterate law $\mu_N$ satisfy $\sqrt m\,W_2(\mu_N,\pi)\le\varepsilon$, for fixed model parameters and initialization. 
We bound the stationary error directly, without assuming third derivatives or a Lipschitz Hessian. Each iteration uses one gradient evaluation and one exact proximal evaluation. The key idea in our analysis is to convert a second-order stationary residual into a Wasserstein bound using a Poisson-based estimate.  
\end{abstract}
\setcounter{tocdepth}{2}
\tableofcontents

\section{Introduction}\label{sec:introduction}

We study a sampling problem from composite log-concave distributions
\[
 \pi(\mathrm dx)\propto e^{-f(x)-g(x)}\dx,
 \qquad x\in\R^d,
\]
where $f\in C^2(\R^d)$ is $m$-strongly convex with $L_f$-Lipschitz gradient, and $g$ is convex and globally $G$-Lipschitz. The nonsmooth term can represent penalties such as $\ell_1$ regularization. The Moreau--Yosida unadjusted Langevin algorithm (MYULA) of \citet{DMP2018} replaces $g$ by its Moreau envelope $g_\lambda$ and uses the update
\[
 \widehat X_{k+1}
 =\widehat X_k-h\nabla(f+g_\lambda)(\widehat X_k)
   +\sqrt{2h}\,\xi_{k+1},
 \qquad \xi_{k+1}\stackrel{\mathrm{i.i.d.}}\sim N(0,I_d).
\]
The identity $\nabla g_\lambda=\lambda^{-1}(I-\prox_{\lambda g})$ allows each iteration to use one evaluation of $\nabla f$ and one exact proximal evaluation of $g$, without an acceptance step.

The central issue is how smoothing affects discretization accuracy. Reducing $\lambda$ improves the approximation of $\pi$ by $\pi_\lambda\propto e^{-f-g_\lambda}$, but increases the global smoothness bound $L_\lambda=L_f+\lambda^{-1}$. Thus, a small-step bias estimate for a fixed $\lambda$ does not by itself give a sharp complexity bound for the original target: its dependence on $\lambda$ must also be controlled.

\subsection{Contribution}
In this work, we would address the above issue and present  error bounds for MYULA.  
We denote
\[
 \tau_f:=\sup_x\tr[\nabla^2f(x)],
 \qquad R:=\frac{G+\sqrt{G^2+4\tau_f}}2,
\]
and let $\ell_{h,\lambda}:=1+\log(e+(mh)^{-1}+L_\lambda/m)$.
Under the step-size condition \eqref{eq:bias-step}, \cref{thm:bias} proves that the MYULA invariant law satisfies
\[
 \sqrt m\,W_2(\widehat\pi_{\lambda,h},\pi_\lambda)
 \le ChR^2\ell_{h,\lambda}^2.
\]
The error coefficient depends on $\lambda^{-1}$ only through a logarithm, although the sufficient step-size restriction is stronger than $hL_\lambda\le c$.

For $0<\varepsilon\le1$, choose $\lambda=\varepsilon/G^2$ and the fixed step size in \cref{thm:complexity}. Combining the invariant-measure estimate with the Moreau bias bound of \citet[Proposition~5.9]{ActiveTrace} gives the sufficient iteration complexity
\[
 N=\widetilde O\!\left[
 \frac{L_f\tau_f}{m^2}+\frac{L_fR}{m^{3/2}}
 +\frac1\varepsilon\left(
       \frac{R^2}{m}+\frac{G^3R}{m^2}
   \right)\right]
\]
for $\sqrt m\,W_2(\mu_N,\pi)\le\varepsilon$, where $\mu_N=\mathcal L(\widehat X_N)$ and $\widetilde O$ suppresses logarithmic factors. For fixed model parameters and initialization, this improves the $\widetilde O(\varepsilon^{-4/3})$ guarantee of \citet{XinZhang} to $\widetilde O(\varepsilon^{-1})$. For fixed $L_f/m$ and $G/\sqrt m$, the bound also gives $\widetilde O(d/\varepsilon)$ iterations.

\paragraph{Proof strategy.}
The core strategy of our analysis is a Poisson-based estimate that converts a second-order stationary residual into a Wasserstein bound (\cref{sec:residual-lemma}).
By testing the residual against a Poisson solution, we place derivatives on the test function rather than on the residual fields.
Weighted energy estimates control the resulting terms while retaining the Hessian as a weight, instead of replacing it by its global bound. In particular, 
for MYULA, exact drift--heat identities express the residual through vector and matrix fluxes (\cref{sec:flux}).
Finite-order density-ratio estimates and short-time heat estimates control these fluxes (\cref{sec:flux-proofs}), allowing the curvature contribution to be bounded through its average under $\pi_\lambda$.
Together, these steps yield a near-linear invariant-measure bias with only logarithmic dependence on $\lambda^{-1}$ in its coefficient, without additional higher-order smoothness assumptions (\cref{sec:bias-proof}).
Adding the smoothing and convergence errors gives the complexity bound in \cref{sec:complexity-proof}.

\subsection{Related Work}\label{sec:related}

\paragraph{MYULA and averaged curvature.}
\citet{DMP2018} introduced MYULA and established asymptotic and nonasymptotic convergence guarantees for nonsmooth log-concave targets. More recently, \citet{ActiveTrace} analyzed the weak Moreau Hessian through its average along a reference heat path. For the structured penalties treated there, curvature localization yields $\widetilde O(\varepsilon^{-2})$ end-to-end complexity. The closest result is \citet{XinZhang}, who consider the same assumptions and MYULA transition as this paper. Their discrete Poisson-corrector analysis gives an invariant-measure bias of $O(h)+\widetilde O(h^{3/4})$, with only logarithmic smoothing dependence in its coefficients, and consequently $\widetilde O(\varepsilon^{-4/3})$ complexity. Our stationary-density analysis replaces this bias bound by $\widetilde O(h)$ under a stronger step-size restriction. The improvement concerns the fixed-parameter accuracy exponent; it does not assert uniformly better dependence on every model parameter.

\paragraph{First-order ULA bounds and the smoothing parameter.}
For a smooth, $m$-strongly convex potential with $L$-Lipschitz gradient, \citet[Theorem~1]{PedrottiWhalley2026} prove
\[
 \sqrt m\,W_2(\widehat\pi_h,\pi)\le6hL\sqrt d,
 \qquad 0<hL\le1.
\]
Applying this estimate to $f+g_\lambda$, with the $C^{1,1}$ case obtained by mollification, and choosing $\lambda\asymp\varepsilon/G^2$ gives the sufficient MYULA complexity
\[
 \widetilde O\!\left[
 \frac{\sqrt d}{m}\left(
 \frac{L_f}{\varepsilon}+\frac{G^2}{\varepsilon^2}
 \right)\right].
\]
Thus, first-order bias for a fixed smooth target does not directly imply near-linear accuracy complexity after Moreau smoothing.
The gain is in the joint smoothing--discretization analysis, not in establishing first-order ULA bias at fixed smoothness.

\paragraph{Other proximal samplers.}
Different transitions can achieve higher accuracy without a fixed Moreau approximation. For example, \citet[Proposition~5]{FanYuanChen2023} obtain polylogarithmic accuracy dependence in $W_2$ for strongly convex composite targets, under their initialization and subroutine assumptions. Their method uses a restricted Gaussian sampling step implemented by approximate rejection sampling, rather than a deterministic proximal update. These guarantees concern a different algorithm and computational model. Our result is an iteration bound for classical MYULA with one exact proximal evaluation per step.

\section{Preliminaries}\label{sec:preliminaries}

All measures are defined on $\R^d$, and $\mathcal P_2(\R^d)$ denotes the probability measures with finite second moment.
For $\mu,\nu\in\mathcal P_2(\R^d)$,
\begin{equation}\label{eq:W2}
 W_2^2(\mu,\nu):=
 \inf_{\mathcal L(X)=\mu,\,\mathcal L(Y)=\nu}\E\norm{X-Y}^2.
\end{equation}
The same symbol denotes a measure and its Lebesgue density when a density exists.
We use the Euclidean norm and inner product for vectors, the Frobenius norm
$\norm M_F^2=\sum_{i,j}M_{ij}^2$ for matrices, and $\ip{M}{N} = \mathrm{tr}(M^\top N)=\sum_{i,j}M_{ij}N_{ij}$.
For symmetric matrices, $M\preceq N$ means $v^\top Mv\le v^\top Nv$ for every $v$.
For a scalar, vector, or matrix field $F$, define
\begin{equation}\label{eq:Lp-def}
 \norm{F}_{L^s(\nu)}
 :=\left(\int\norm{F(x)}^s\nu(\mathrm dx)\right)^{1/s},
 \qquad 1\le s<\infty.
\end{equation}
The pointwise norm in \eqref{eq:Lp-def} is the absolute value for scalars, the Euclidean norm for vectors, and the Frobenius norm for matrices. Thus, for a matrix field $F$,
\[
 \norm{F}_{L^2(\nu)}^2
 =\int\norm{F(x)}_F^2\nu(\mathrm dx)
 =\sum_{i,j}\int|F_{ij}(x)|^2\nu(\mathrm dx).
\]
In particular, for a positive density $\nu$ and a scalar, vector, or matrix density $F$,
\[
 \norm{F/\nu}_{L^s(\nu)}^s
   =\int\norm{F(x)}^s\nu(x)^{1-s}\dx.
\]
For $v,w\in\R^d$, the outer product is the matrix
\begin{equation}\label{eq:outer}
 (v\otimes w)_{ij}:=v_iw_j,
 \qquad \norm{v\otimes w}_F=\norm v\norm w.
\end{equation}

\subsection{Weak derivatives}\label{sec:notation}

For a locally integrable function $u$, its distributional derivative is defined by
\[
 \langle\partial_i u,\phi\rangle=-\int u\,\partial_i\phi\dx,
 \qquad \phi\in C_c^\infty(\R^d),
\]
where $C_c^\infty$ denotes smooth functions with compact support.
If this derivative is represented by a locally integrable function $v$, then $v$ is the weak derivative of $u$.
Lipschitz functions have almost-everywhere derivatives, meaning their weak derivatives.
An equality of distributions implies an equality against every such test function.
When both sides are locally integrable functions, this is equivalent to equality almost everywhere; it does not assert the existence of classical derivatives at every point.

For a vector field $H$ and a matrix field $E$, our divergence convention is
\begin{equation}\label{eq:div}
 \Div H=\sum_i\partial_iH_i,
 \qquad (\Div E)_i=\sum_j\partial_jE_{ij},
 \qquad \Div\Div E=\sum_{i,j}\partial_i\partial_jE_{ij}.
\end{equation}
In particular,
\begin{equation}\label{eq:weak-div}
 \langle\Div H,\phi\rangle=-\int\ip{\nabla\phi}{H}\dx,
 \qquad
 \langle\Div\Div E,\phi\rangle=\int \langle E,\nabla^2 \phi \rangle \dx.
\end{equation}
These formulas require only local integrability of the fields.
They allow density equations to be used through integration by parts without differentiating the densities classically.
Hessians are written explicitly as $\nabla^2 f$, $\nabla^2g_\lambda$, or $\nabla^2U_\lambda$; the letter $H$ denotes a vector flux.
Constants $c,C>0$ are universal and may change from line to line.
The notation $\widetilde O$ suppresses fixed powers of logarithms, but no polynomial parameter factors.

\subsection{Moreau regularization}\label{sec:moreau}
For a convex, globally $G$-Lipschitz function $g:\R^d\to\R$ and $\lambda>0$, let
\begin{equation}\label{eq:moreau-def}
 g_\lambda(x):=\inf_{y\in\R^d}
 \left\{g(y)+\frac{\norm{x-y}^2}{2\lambda}\right\}.
\end{equation}
The unique minimizer is denoted by $\prox_{\lambda g}(x)$.
The Moreau properties used below are
\begin{equation}\label{eq:moreau-properties}
 \nabla g_\lambda(x)=\lambda^{-1}\bigl(x-\prox_{\lambda g}(x)\bigr),
 \quad \norm{\nabla g_\lambda(x)}\le G,
 \quad 0\preceq\nabla^2g_\lambda\preceq\lambda^{-1}I\quad\text{a.e.}
\end{equation}
In particular, $g_\lambda\in C^{1,1}$, meaning that its gradient is globally Lipschitz.
The last inequality concerns its weak Hessian, not a continuous second derivative.
We use these facts in the form recalled in \citet[Section~3.2 and Appendix~A]{XinZhang}.

\subsection{Drift maps, pushforwards, and density evolution}\label{sec:operators-def}
The potential used for MYULA will be specified in \cref{sec:setup}.
For the remainder of this section, we use only its regularity and convexity properties: fix $\lambda>0$ and let $U_\lambda\in C^{1,1}(\R^d)$ be $m$-strongly convex with $L_\lambda$-Lipschitz gradient, where $0<m\le L_\lambda<\infty$.
Write $\pi_\lambda$ for the probability density proportional to $e^{-U_\lambda}$.
Its weak Hessian satisfies
\begin{equation}\label{eq:potential-hessian-bounds}
 mI\preceq\hess\preceq L_\lambda I
 \qquad\text{almost everywhere}.
\end{equation}
The deterministic drift step is
\begin{equation}\label{eq:T}
 T_t(x):=x-t\drift(x).
\end{equation}
For an integrable scalar, vector, or matrix density $F$, the pushforward of $F\dx$ assigns to each Borel set the integral of $F$ over its inverse image under $T_t$.
This defines a unique weighted measure, componentwise.
If it has a Lebesgue density, we denote that density by $(T_t)_\#F$; equivalently,
\begin{equation}\label{eq:push-def}
 \int\phi(y)(T_t)_\#F(y)\dd y
 =\int\phi(T_tx)F(x)\dx
 \quad\text{for every }\phi\in C_c^\infty(\R^d).
\end{equation}
For a probability density, the pushforward is the law of the transformed random variable $T_tX$.
For vector and matrix densities, the pushforward is defined componentwise.

\begin{proposition}[Drift pushforward]\label{prop:pushforward}
If $0\le tL_\lambda<1$, then $T_t$ is a global bi-Lipschitz bijection and
\[
 \operatorname{Lip}(T_t^{-1})\le(1-tL_\lambda)^{-1}.
\]
For every integrable scalar, vector, or matrix density $F$, its pushforward has the density
\begin{equation}\label{eq:push-density}
 (T_t)_\#F(y)
 =\frac{F(T_t^{-1}y)}
 {\det\!\left(I-t\nabla^2U_\lambda(T_t^{-1}y)\right)}
 \quad\text{for a.e. }y.
\end{equation}
In particular, $\norm{(T_t)_\#F}_{L^1(\R^d)}=\norm F_{L^1(\R^d)}$.
\end{proposition}
\begin{proof}
For each $y$, the map $x\mapsto y+t\drift(x)$ is a contraction on $\R^d$.
Its unique fixed point is $T_t^{-1}y$.
The fixed-point equations for two values of $y$ give the stated inverse Lipschitz bound, and $T_t$ is Lipschitz by definition.
Its derivative is $DT_t=I-t\hess$ almost everywhere, with
\[
 \det(I-t\hess)\ge(1-tL_\lambda)^d>0.
\]
The weighted area formula \citep[Chapter~2, Theorem~3.3 and Remark~3.4(2)]{SimonGMT} applied to the Lipschitz bijection $T_t$ gives
\[
 \int q(x)\det(I-t\nabla^2U_\lambda(x))\dx
 =\int q(T_t^{-1}y)\dd y
\]
for every nonnegative measurable scalar function $q$.
Applying this formula with the integrand divided by the positive Jacobian, and then taking positive and negative parts componentwise,
yields
\[
 \int\phi(T_tx)F(x)\dx
 =\int\phi(y)
 \frac{F(T_t^{-1}y)}
 {\det(I-t\nabla^2U_\lambda(T_t^{-1}y))}\dd y.
\]
Both $T_t$ and $T_t^{-1}$ preserve Lebesgue null sets because they are Lipschitz, so the almost-everywhere derivatives suffice for these formulas. The last identity proves \eqref{eq:push-density}.
Changing variables in its pointwise norm gives the stated $L^1$ identity.
\end{proof}

The free heat semigroup is
\begin{equation}\label{eq:heat-def}
 \heat_tF(x):=(4\pi t)^{-d/2}
 \int e^{-\norm{x-y}^2/(4t)}F(y)\dd y,
 \qquad t>0,\qquad \heat_0F=F.
\end{equation}
It acts componentwise and adds an independent $N(0,2tI_d)$ increment to a probability density.
For $F\in L^1(\R^d)$, it preserves the integral and satisfies $\norm{\heat_tF}_{L^1}\le\norm F_{L^1}$.
Moreover, $\heat_tF\to F$ in $L^1$ as $t\downarrow0$, since the Gaussian kernels form an approximate identity.

We consider a gradient bound and an integrated form of the heat equation.
For scalar $F\in L^1(\R^d)$, differentiation of the Gaussian kernel, the triangle inequality, and Tonelli's theorem give
\begin{align}\label{eq:heat-basic}
 \norm{\nabla\heat_tF}_{L^1(\R^d)}
 &\le \norm F_{L^1(\R^d)}
       \int_{\R^d}\frac{\norm z}{2t}
       (4\pi t)^{-d/2}e^{-\norm z^2/(4t)}\dd z\notag\\
 &\le \frac{C_d}{\sqrt t}\norm F_{L^1(\R^d)},
 \qquad C_d:=\sqrt{d/2}.
\end{align}
The second inequality is Cauchy--Schwarz and the second moment $2td$ of this Gaussian density.
In particular, $t\norm{\nabla\heat_tF}_{L^1}$ is integrable near $t=0$.

For $\phi\in C_c^\infty(\R^d)$, symmetry of the Gaussian kernel gives
\[
 \int\phi\heat_tF\dx=\int F\heat_t\phi\dx.
\]
For $t>0$, differentiation of the kernel and integration by parts against the smooth function $\phi$ yield $\partial_t\heat_t\phi=\heat_t\Delta\phi$.
Since
\[
 \left|F(x)\partial_t\heat_t\phi(x)\right|
 \le |F(x)|\norm{\Delta\phi}_\infty,
\]
differentiation under the integral is justified by dominated convergence. Hence
\begin{equation}\label{eq:heat-weak-time}
 \frac\dd{\dd t}\int\phi\heat_tF\dx
 =\int(\Delta\phi)\heat_tF\dx,
 \qquad t>0.
\end{equation}
The derivative is bounded by $\norm F_{L^1}\norm{\Delta\phi}_\infty$.
Integrating first over a positive-time interval and then using $\heat_tF\to F$ in $L^1$ at the left endpoint gives
\begin{equation}\label{eq:heat-weak-integrated}
 \int\phi(\heat_bF-\heat_aF)\dx
 =\int_a^b\int(\Delta\phi)\heat_tF\dx\dd t,
 \qquad 0\le a\le b<\infty.
\end{equation}
Thus, for $F\in L^1(\R^d)$ and $\phi\in C_c^\infty(\R^d)$, the map $t\mapsto\int\phi\heat_tF\dx$ is absolutely continuous on $[0,T]$ for every $T>0$.

The Langevin generator and its formal adjoint with respect to Lebesgue measure are defined by
\begin{equation}\label{eq:gen}
 \gen\phi=\Delta\phi-\ip{\drift}{\nabla\phi},
 \qquad
 \gen^*\rho=\Delta\rho+\Div(\rho\drift).
\end{equation}
Here $\gen^*\rho$ is understood distributionally: $\langle\gen^*\rho,\phi\rangle=\int\rho\gen\phi\dx$ for $\phi\in C_c^\infty(\R^d)$.
Because $\nabla\pi_\lambda=-\pi_\lambda\drift$, we have $\gen^*\pi_\lambda=0$.

\subsection{Poisson equations and transport}\label{sec:poisson-background}

The Poisson equation will be used to convert a residual in the stationary density equation into a difference of expectations.
The weighted Sobolev space $H^1(\pi_\lambda)$ consists of functions whose value and weak gradient belong to $L^2(\pi_\lambda)$.
Strong convexity gives the Poincar\'e inequality
\begin{equation}\label{eq:Poincare}
 \int\left|v-\int v\dd\pi_\lambda\right|^2\dd\pi_\lambda
 \le\frac1m\int\norm{\nabla v}^2\dd\pi_\lambda;
\end{equation}
see, for example, \citet{Chewi}.
For $\phi\in L^2(\pi_\lambda)$, the mean-zero Poisson equation is
\begin{equation}\label{eq:poisson-def}
 \gen\psi=\phi-\int\phi\dd\pi_\lambda,
 \qquad \int\psi\dd\pi_\lambda=0.
\end{equation}
A weak solution is a function $\psi\in H^1(\pi_\lambda)$ satisfying
\begin{equation}\label{eq:poisson-weak}
 \int\ip{\nabla\psi}{\nabla v}\dd\pi_\lambda
 =-\int\left(\phi-\int\phi\dd\pi_\lambda\right)v\dd\pi_\lambda,
 \qquad v\in H^1(\pi_\lambda).
\end{equation}
Here $\phi$ is prescribed and $\psi$ is the unknown function; the zero-mean condition removes the freedom to add a constant.
For $k=1,2$, the unweighted Sobolev space $H^k(\R^d)$ consists of functions whose weak derivatives of order at most $k$ belong to $L^2(\R^d)$ with respect to Lebesgue measure. We write $u\in H^2_{\mathrm{loc}}(\R^d)$ if $\eta u\in H^2(\R^d)$ for every $\eta\in C_c^\infty(\R^d)$.
The next lemma shows the regularity and approximation needed below.

\begin{lemma}[Poisson regularity and approximation]\label{lem:poisson-energy}
Let $U_\lambda\in C^{1,1}(\R^d)$ be $m$-strongly convex with $L_\lambda$-Lipschitz gradient, where $0<m\le L_\lambda<\infty$. Let $\pi_\lambda\propto e^{-U_\lambda}$ and let $\gen$ be the generator in \eqref{eq:gen}.
For every $\phi\in L^2(\pi_\lambda)$, the Poisson equation \eqref{eq:poisson-def} has a unique mean-zero weak solution $\psi\in H^1(\pi_\lambda)$. Moreover, $\psi\in H^2_{\mathrm{loc}}(\R^d)$.
If $\phi\in C_c^\infty(\R^d)$, then $\nabla^2\psi\in L^2(\pi_\lambda)$ and
\begin{align}\label{eq:bochner-main}
 &\int\norm{\nabla^2\psi}_F^2\dd\pi_\lambda
   +\int\nabla\psi^\top(\hess)\nabla\psi\dd\pi_\lambda\notag\\
 &\qquad\le\int\left|\phi-\int\phi\dd\pi_\lambda\right|^2
                 \dd\pi_\lambda
 \le\frac1m\int\norm{\nabla\phi}^2\dd\pi_\lambda.
\end{align}
For such $\phi$, there exist $\psi_n\in C_c^\infty(\R^d)$ such that
\begin{equation}\label{eq:poisson-approximation}
 \begin{split}
 \Lp{\gen\psi_n-\gen\psi}{2}
  +\Lp{\nabla^2\psi_n-\nabla^2\psi}{2}
 +\Lp{\nabla\psi_n-\nabla\psi}{2}\longrightarrow0.
 \end{split}
\end{equation}
\end{lemma}
The proof is given in \cref{app:poisson}.
The estimate \eqref{eq:bochner-main} controls the weak derivatives of the Poisson solution, whereas \eqref{eq:poisson-approximation} allows it to be used in residual identities initially stated only for compactly supported smooth tests.
Neither assertion requires a continuous Hessian or a third derivative of $U_\lambda$.

For an integrable signed density $\sigma$ with $\int\sigma\dx=0$,
define
\begin{equation}\label{eq:Hminus-def}
 \Nminus{\sigma}
 :=
 \sup_{\substack{\phi\in C_c^\infty(\R^d)\\
                  \Lp{\nabla\phi}{2}\le1}}
 \left|\int\phi\sigma\dx\right|.
\end{equation}
Here $\sigma$ represents the signed measure with Lebesgue density $\sigma$, not a density relative to $\pi_\lambda$.
In particular, the function--measure pairing is
\begin{equation}\label{eq:Hminus-measure-pairing}
 \langle\phi,\sigma(x)\dx\rangle
 =\int\phi\sigma\dx
 =\int\phi\,\frac{\sigma}{\pi_\lambda}\dd\pi_\lambda.
\end{equation}

If $\sigma/\pi_\lambda\in L^2(\pi_\lambda)$, the same norm is obtained using either weighted Sobolev tests or $C^1$ tests:
\begin{align}\label{eq:Hminus-test-classes}
 \Nminus{\sigma}
 =
 \sup_{\substack{\phi\in H^1(\pi_\lambda)\\
                  \Lp{\nabla\phi}{2}\le1}}
 \left|\int\phi\sigma\dx\right|
 =
 \sup_{\substack{\phi\in C^1(\R^d)\\
                  \Lp{\nabla\phi}{2}\le1}}
 \left|\int\phi\sigma\dx\right|.
\end{align}
The approximation and integrability needed for these equalities are proved in \cref{app:negative-norm}. Consequently, \eqref{eq:Hminus-def} agrees, in this setting, with the norm of the signed measure $\sigma(x)\dx$ in \citet[Section~1.1]{Peyre}.

For a probability density $\nu\in\mathcal P_2(\R^d)$ with $\nu/\pi_\lambda\in L^2(\pi_\lambda)$, apply \citet[Theorem~1]{Peyre} to the measures $\pi_\lambda(x)\dx$ and $\nu(x)\dx$. By \eqref{eq:Hminus-measure-pairing} and \eqref{eq:Hminus-test-classes}, this gives
\begin{equation}\label{eq:transport-comparison}
 W_2(\nu,\pi_\lambda)\le2\Nminus{\nu-\pi_\lambda}.
\end{equation}

\subsection{Entropy and density-ratio moments}\label{sec:entropy}
For $F\ge0$, write
\[
 \Ent_\nu F:=\int F\log F\dd\nu
       -\left(\int F\dd\nu\right)\log\int F\dd\nu.
\]
Strong convexity also gives the log-Sobolev inequality \citep{Chewi}
\begin{equation}\label{eq:LSI}
 \Ent_{\pi_\lambda}(w^2)
 \le\frac2m\int\norm{\nabla w}^2\dd\pi_\lambda.
\end{equation}
For $\alpha>1$ and a probability density $\nu$, the R\'enyi divergence is
\begin{equation}\label{eq:renyi-def}
 \mathcal R_\alpha(\nu\Vert\pi_\lambda)
 :=\frac1{\alpha-1}\log\int
              \left(\frac\nu{\pi_\lambda}\right)^\alpha\dd\pi_\lambda.
\end{equation}
Equivalently,
$\norm{\nu/\pi_\lambda}_{L^\alpha(\pi_\lambda)}
=\exp\{(\alpha-1)\mathcal R_\alpha(\nu\Vert\pi_\lambda)/\alpha\}$.

\section{The Moreau–Yosida Unadjusted Langevin Sampling} \label{sec:myula}

In this section we study theoretical analysis for error bounds of the Moreau–Yosida unadjusted Langevin algorithm (MYULA). We first give some assumptions with the problem description, and then present our main theoretical results.

\subsection{Problem formulation and assumptions}\label{sec:setup}

We consider the probability distribution
\begin{equation}\label{eq:target}
 \pi(\mathrm dx)=Z^{-1}\exp\{-f(x)-g(x)\}\dx,
 \qquad x\in\R^d.
\end{equation}
The assumptions and normalization below follow those of \citet{XinZhang}.

\begin{assumption}[Smooth component]\label{ass:f}
The function $f\in C^2(\R^d)$ satisfies
\begin{equation}\label{eq:f-ass}
 mI\preceq\nabla^2 f(x)\preceq L_fI,
 \qquad x\in\R^d,
 \qquad 0<m\le L_f<\infty.
\end{equation}
Write
\begin{equation}\label{eq:trace-f}
 \tau_f:=\sup_{x\in\R^d}\tr [\nabla^2f(x)].
\end{equation}
In particular, $md\le\tau_f\le dL_f$.
\end{assumption}

\begin{assumption}[Nonsmooth component]\label{ass:g}
The function $g\colon \R^d\to\R$ is convex and globally $G$-Lipschitz, for some $G>0$.
\end{assumption}

For the Moreau envelope in \eqref{eq:moreau-def}, set
\begin{equation}\label{eq:smoothed}
 U_\lambda:=f+g_\lambda,\qquad
 \pi_\lambda(\mathrm dx):=Z_\lambda^{-1}e^{-U_\lambda(x)}\dx,
 \qquad L_\lambda:=L_f+\lambda^{-1}.
\end{equation}
By \cref{ass:f,ass:g} and \eqref{eq:moreau-properties}, this choice of $U_\lambda$ belongs to $C^{1,1}(\R^d)$ and satisfies \eqref{eq:potential-hessian-bounds}.
Thus the operator constructions and Poisson results in \cref{sec:operators-def,sec:poisson-background} apply.

MYULA is the fixed-step recursion
\begin{equation}\label{eq:MYULA}
 \widehat X_{k+1}=\widehat X_k-h\drift(\widehat X_k)
                      +\sqrt{2h}\,\xi_{k+1},
 \qquad \xi_{k+1}\stackrel{\mathrm{i.i.d.}}\sim N(0,I_d).
\end{equation}
Its transition kernel is $Q_{\lambda,h}$, and $\mu_k:=\mathcal L(\widehat X_k)$.
Under the step-size condition below, its invariant law is denoted by $\inv$.
The initial law $\mu_0$ belongs to $\mathcal P_2(\R^d)$.
We measure accuracy by $\sqrt m\,W_2(\mu_k,\pi)$.
Each iteration uses one evaluation of $\nabla f$ and one exact proximal evaluation of $g$.

Let
\begin{equation}\label{eq:R}
 R:=\frac{G+\sqrt{G^2+4\tau_f}}2.
\end{equation}
Thus, $R^2=\tau_f+GR$, $R\ge G$, and $R^2\ge\tau_f\ge md$.

\begin{proposition}[Invariant law and contraction]\label{prop:invariant}
Under \cref{ass:f,ass:g}, if $0<hL_\lambda<1$, then $Q_{\lambda,h}$ has a unique invariant law in $\mathcal P_2(\R^d)$, with a smooth strictly positive density $\inv$.
For $\mu,\nu\in\mathcal P_2(\R^d)$,
\begin{equation}\label{eq:kernel-contraction}
 W_2(\mu Q_{\lambda,h},\nu Q_{\lambda,h})
 \le(1-mh)W_2(\mu,\nu).
\end{equation}
Moreover,
\begin{equation}\label{eq:stationarity}
 \inv=\heat_h((T_h)_\#\inv).
\end{equation}
\end{proposition}
\begin{proof}
For a smooth potential, integration of the Hessian along a segment gives $\norm{T_hx-T_hy}\le(1-mh)\norm{x-y}$. Mollification preserves the Hessian bounds and gives the same inequality for $U_\lambda\in C^{1,1}$ by uniform convergence of its gradient. Using the same Gaussian increment for two Euler updates proves \eqref{eq:kernel-contraction}.
The linear growth of the drift implies that the kernel maps $\mathcal P_2$ into itself. The contraction theorem on the complete metric space $(\mathcal P_2,W_2)$ therefore gives the unique invariant law in that space.
An Euler update consists of the pushforward by $T_h$ followed by the heat operator $\heat_h$, which proves \eqref{eq:stationarity}.
Convolution of a probability measure with the Gaussian kernel has a smooth strictly positive density.
\end{proof}

\subsection{Main results}\label{sec:results}

The first theorem controls the invariant-measure bias at fixed smoothing; the second gives a single step size for the original target. Set
\begin{align}
 \ell_{h,\lambda}&:=1+\log\left(e+\frac1{mh}+\frac{L_\lambda}{m}\right),
                                                        \label{eq:log-h}\\
 \ell_\varepsilon&:=1+\log\left(e+\frac{L_f}{m}+\frac{\tau_f}{m}
                              +\frac{G^2}{m}+\frac1\varepsilon\right).
                                                        \label{eq:log-eps}
\end{align}

\begin{theorem}[Invariant-measure bias at fixed smoothing]\label{thm:bias}
Under \cref{ass:f,ass:g}, there exist universal constants $c,C>0$ such that, for every $\lambda,h>0$ satisfying
\begin{equation}\label{eq:bias-step}
 h\ell_{h,\lambda}^{\,2}
 \left[\frac{L_f\tau_f}{m}+\frac{L_fR}{\sqrt m}
       +\frac{GR}{m\lambda}+R^2+\frac1\lambda\right]\le c,
\end{equation}
the kernel $Q_{\lambda,h}$ has a unique invariant law $\inv\in\mathcal P_2(\R^d)$, and
\begin{equation}\label{eq:bias}
 \sqrt m\,W_2(\inv,\pi_\lambda)
 \le ChR^2\ell_{h,\lambda}^{\,2}.
\end{equation}
\end{theorem}

The proof of \cref{thm:bias} is given in \cref{sec:bias-proof}.
The sufficient restriction \eqref{eq:bias-step} is stronger than $hL_\lambda\le c$.

\begin{theorem}[End-to-end complexity]\label{thm:complexity}
Under \cref{ass:f,ass:g}, let $0<\varepsilon\le1$ and choose
\begin{align}
 \lambda_\varepsilon&:=\frac{\varepsilon}{G^2},\label{eq:lambda-eps}\\
 h_\varepsilon&:=\frac{c}{\ell_\varepsilon^2}
 \left[\frac{L_f\tau_f}{m}+\frac{L_fR}{\sqrt m}
             +\frac{R^2+G^3R/m}{\varepsilon}\right]^{-1},\label{eq:h-eps}
\end{align}
where $c>0$ is a sufficiently small universal constant.
\begin{samepage}
For every $\mu_0\in\mathcal P_2(\R^d)$, if
\begin{equation}\label{eq:N-choice}
 N\ge\frac{C}{mh_\varepsilon}
       \log\left(2+\frac{\sqrt m\,W_2(\mu_0,\pi)}{\varepsilon}\right),
\end{equation}
then
\begin{equation}\label{eq:accuracy}
 \sqrt m\,W_2\bigl(\mu_0Q_{\lambda_\varepsilon,h_\varepsilon}^N,\pi\bigr)
 \le\varepsilon.
\end{equation}
\end{samepage}
Moreover, the smallest integer satisfying \eqref{eq:N-choice} obeys
\begin{align}\label{eq:N-explicit}
 N\le 1+\frac{C\ell_\varepsilon^2}{m}
 \left[\frac{L_f\tau_f}{m}+\frac{L_fR}{\sqrt m}
       +\frac{R^2+G^3R/m}{\varepsilon}\right]
 \log\left(2+\frac{\sqrt m\,W_2(\mu_0,\pi)}{\varepsilon}\right).
\end{align}
In particular, for fixed model parameters and fixed initialization,
$N(\varepsilon)=\widetilde O(\varepsilon^{-1})$.
\end{theorem}

The proof of \cref{thm:complexity} is given in \cref{sec:complexity-proof}.
The smoothing step uses the Moreau approximation estimate of \citet[Proposition~5.9]{ActiveTrace}, also used in \citet[Section~6.5]{XinZhang}:
\begin{equation}\label{eq:moreau-bias}
 \sqrt m\,W_2(\pi_\lambda,\pi)\le\frac{G^2\lambda}{4}.
\end{equation}
The choices \eqref{eq:lambda-eps}--\eqref{eq:h-eps} leave this bias budget unchanged.

\begin{remark}[Parameter dependence]\label{rem:parameters}
Using $R\le G+\sqrt{\tau_f}$, the sufficient iteration bound \eqref{eq:N-explicit} implies
\begin{equation}\label{eq:N-trace}
 N=\widetilde O\left[
 \frac{L_f\tau_f}{m^2}+\frac{L_fG}{m^{3/2}}
 +\frac1\varepsilon\left(
 \frac{\tau_f}{m}+\frac{G^3\sqrt{\tau_f}}{m^2}
 +\frac{G^4}{m^2}
 \right)\right],
\end{equation}
where $\widetilde O$ suppresses the logarithmic factors in \eqref{eq:N-explicit}. For fixed $L_f/m$ and $G/\sqrt m$, the inequality $\tau_f\le dL_f$ gives $N=\widetilde O(d/\varepsilon)$.
\end{remark}

\section{Proof of the Main Results}\label{sec:proof}
This section proves \cref{thm:bias,thm:complexity}.
We first use stationarity and a Poisson equation to bound $W_2(\inv,\pi_\lambda)$, using the flux estimates proved in \cref{sec:flux-proofs}.
We then combine this bound with the Moreau approximation error and contraction of the chain to obtain the iteration bound.

\subsection{The analytic estimates}\label{sec:inputs}

\begin{lemma}[Target drift moments]\label{lem:moments}
Under \cref{ass:f,ass:g},
\begin{align}
 \int\norm{\drift}^2\dd\pi_\lambda
 &=\int\tr(\hess)\dd\pi_\lambda\le R^2,\label{eq:target-second}\\
 \Lp{\drift}{s}&\le3\sqrt s\,R,\qquad s\ge2,\label{eq:target-high}\\
 \int\exp\left(\frac{\norm{\drift}^2}{256R^2}\right)\dd\pi_\lambda&\le2.
                                                        \label{eq:target-exp}
\end{align}
\end{lemma}
The second-moment identity and bound follow from \citet[Lemma~7.1 and Proposition~A.4]{XinZhang}.
The higher-moment and exponential estimates are proved in \cref{app:moments}.
These estimates control the mean curvature and the drift moments using $G$ and $\tau_f$, uniformly in $\lambda$.

\begin{lemma}[Finite-order stationary density ratios]\label{lem:ratios}
Let $\alpha\ge2$. Suppose $hL_\lambda\le c$ and
\begin{equation}\label{eq:ratio-step}
 h\frac{L_f\tau_f}{m}\le\frac c{\alpha^2},\qquad
 h\frac{L_fR}{\sqrt m}\le\frac c\alpha,\qquad
 \frac h\lambda\frac{GR}{m}\le\frac c{\alpha^2}.
\end{equation}
Then the stationary Euler interpolation
\begin{equation}\label{eq:Euler-interpolation}
 \nu_t:=\heat_t((T_t)_\#\inv),\qquad 0\le t\le h,
 \qquad \nu_0=\nu_h=\inv,
\end{equation}
satisfies
\begin{equation}\label{eq:ratios}
 \sup_{0\le t\le h}\Lp{\nu_t/\pi_\lambda}{\alpha}\le e^{1/4}.
\end{equation}
\end{lemma}
The proof is given in \cref{app:renyi}.
The prior integrability needed for differentiation is established independently in \cref{app:finiteness}.
The equal endpoint laws close a R\'enyi dissipation estimate over one Euler step using the log-Sobolev inequality for $\pi_\lambda$.

\begin{lemma}[Normalized heat and pushforward estimates]\label{lem:operators}
The estimates below apply to measurable scalar, vector, or matrix densities $F$, as indicated, whenever the norm on the right-hand side is finite. For $s\ge2$, $t>0$, and $ts^2R^2\le c$,
\begin{align}
 \Lp{\heat_tF/\pi_\lambda}{s}
   &\le C\Lp{F/\pi_\lambda}{2s},\label{eq:op-zero}\\
 \Lp{\nabla\heat_tF/\pi_\lambda}{s}
   &\le\frac C{\sqrt t}\Lp{F/\pi_\lambda}{2s}
        &&\text{for scalar }F,\label{eq:op-gradient}\\
 \Lp{\nabla\Div\heat_tF/\pi_\lambda}{s}
   &\le\frac{C\sqrt s}{t}\Lp{F/\pi_\lambda}{4s}
        &&\text{for vector }F.\label{eq:op-second}
\end{align}
Moreover, if $st(L_\lambda+R^2)\le c$, then
\begin{equation}\label{eq:op-push}
 \Lp{(T_t)_\#F/\pi_\lambda}{s}\le C\Lp{F/\pi_\lambda}{2s}.
\end{equation}
The zero-order and pushforward estimates apply to scalar, vector, and matrix densities and extend to $t=0$.
\end{lemma}
The proof is given in \cref{app:heat,app:heat-derivatives,app:push}.

\subsection{The stationary residual}
\label{sec:flux}\label{sec:identities-proof}

Assume $hL_\lambda<1$ and let $\inv$ be the invariant density from \cref{prop:invariant}.
Starting from $\inv$, a full drift step followed by heat evolution for time $h$ returns to $\inv$.
We compare this endpoint with the average along the heat path.
Define
\begin{equation}\label{eq:J}
 J:=\frac1h\int_0^h(T_s)_\#(\drift\,\inv)\dd s,
\end{equation}
and
\begin{equation}\label{eq:heat-interpolation}
 \rho_t:=\heat_t((T_h)_\#\inv),\qquad
 \bar\rho:=\frac1h\int_0^h\rho_t\dd t,
 \qquad 0\le t\le h.
\end{equation}
By stationarity,
\[
 \rho_0=(T_h)_\#\inv,\qquad \rho_h=\inv.
\]
Unlike the Euler interpolation in \eqref{eq:Euler-interpolation}, this path keeps $T_h$ fixed as $t$ varies.
Define also
\begin{align}
 H&:=\frac1h\int_0^h t\nabla\rho_t\dd t,\label{eq:H}\\
 E&:=-\frac1h\int_0^h(h-s)
      (T_s)_\#\bigl((\drift\otimes\drift)\inv\bigr)\dd s.
 \label{eq:E}
\end{align}
The gradient in \eqref{eq:H} is used only for $t>0$.
The following lemma proves that these time integrals define
$L^1$ fields and identifies their weak derivatives.

\begin{lemma}[Exact stationary identities]\label{lem:exact-identities}
Under $hL_\lambda<1$, the fields in \eqref{eq:J}--\eqref{eq:E}
are well defined in $L^1(\R^d)$ and satisfy
\begin{equation}\label{eq:flux-identities}
 \begin{aligned}
 \inv-\bar\rho&=\Div H,\\
 \drift\,\inv-J&=\Div E,\\
 \Delta\bar\rho+\Div J&=0.
 \end{aligned}
\end{equation}
Consequently,
\begin{equation}\label{eq:exact-residual}
 \gen^*(\bar\rho-\pi_\lambda)
 =\Div\Div(E-\drift\otimes H)+\Div((\hess)H).
\end{equation}
All these identities hold in the sense of distributions.
In particular, $\Delta\bar\rho$ denotes the distributional
Laplacian; it has the $L^1$ representative
$(\inv-\rho_0)/h$.
\end{lemma}

\begin{proof}
\emph{Integrability of the fields.}
Since $\drift$ has at most linear growth and
$\inv\in\mathcal P_2(\R^d)$,
\begin{equation}\label{eq:invariant-drift-integrability}
 \int\bigl(\norm{\drift}+\norm{\drift}^2\bigr)\inv\dx<\infty.
\end{equation}
The pushforward and heat-kernel formulas give jointly measurable
versions of the integrands in \eqref{eq:J}--\eqref{eq:E}.
Since $\rho_0$ is a probability density, so are $\rho_t$ and
$\bar\rho$.
By \cref{prop:pushforward} and \eqref{eq:heat-basic},
\begin{align*}
 \norm J_{L^1}
 &\le\int\norm{\drift}\inv\dx,\\
 \norm E_{L^1}
 &\le\frac h2\int\norm{\drift}^2\inv\dx,\\
 \norm H_{L^1}
 &\le\frac1h\int_0^h t\norm{\nabla\rho_t}_{L^1}\dd t
 \le\frac{C_d}{h}\int_0^h\sqrt t\dd t
 =\frac{2C_d}{3}\sqrt h.
\end{align*}
All $L^1$ norms here are with respect to Lebesgue measure.
These estimates justify the time integrals and their pairing with bounded test functions.

\emph{The drift balance.}
We first prove $\rho_0-\inv=h\Div J$.
Fix $\phi\in C_c^\infty(\R^d)$.
For $\varepsilon\ne0$ with $s,s+\varepsilon\in[0,h]$, the
mean value theorem gives
\[
 \left|\frac{\phi(T_{s+\varepsilon}x)-\phi(T_sx)}{\varepsilon}\right|
 \inv(x)
 \le\norm{\nabla\phi}_\infty\norm{\drift(x)}\inv(x).
\]
The right-hand side is integrable by \eqref{eq:invariant-drift-integrability} and independent of $s,\varepsilon$.
Dominated convergence therefore permits differentiation under the integral in \eqref{eq:push-def}, giving
\begin{align}\label{eq:drift-weak-time}
 \frac\dd{\dd s}\int\phi(y)(T_s)_\#\inv(y)\dd y
 &=-\int\ip{\nabla\phi(T_sx)}{\drift(x)}\inv(x)\dx\notag\\
 &=-\int\ip{\nabla\phi}{(T_s)_\#(\drift\,\inv)}\dx.
\end{align}
The same domination gives continuity of this derivative in $s$.
Integrating \eqref{eq:drift-weak-time} over $[0,h]$ yields
\[
 \int\phi(\rho_0-\inv)\dx
 =-h\int\ip{\nabla\phi}{J}\dx
 =h\langle\Div J,\phi\rangle.
\]
Thus
\begin{equation}\label{eq:drift-balance}
 \rho_0-\inv=(T_h)_\#\inv-\inv=h\Div J
 \quad\text{in the sense of distributions}.
\end{equation}

\emph{The heat balance and the endpoint difference.}
To prove the third identity in \eqref{eq:flux-identities}, we compute the distributional Laplacian through its definition.
For $\phi\in C_c^\infty(\R^d)$, Fubini's theorem applies because
\[
 \int_0^h\int\rho_t|\Delta\phi|\dx\dd t
 \le h\norm{\Delta\phi}_\infty<\infty.
\]
The definition of the distributional Laplacian and \eqref{eq:heat-weak-integrated} therefore give
\begin{align}\label{eq:average-heat-balance}
 \langle\Delta\bar\rho,\phi\rangle
 &:=\int\bar\rho\Delta\phi\dx
 =\frac1h\int_0^h\int\rho_t\Delta\phi\dx\dd t\notag\\
 &=\frac1h\int\phi(\rho_h-\rho_0)\dx\notag\\
 &=\frac1h\int\phi(\inv-\rho_0)\dx
 =-\langle\Div J,\phi\rangle.
\end{align}
Here \eqref{eq:heat-weak-integrated} is used with $F=\rho_0$, $a=0$, and $b=h$, and the last equality is \eqref{eq:drift-balance}.
This proves $\Delta\bar\rho+\Div J=0$ and the claimed $L^1$ representative of $\Delta\bar\rho$.
No classical Laplacian of $\bar\rho$ was assumed.

We next prove $\inv-\bar\rho=\Div H$.
The integrability already established for $t\nabla\rho_t$ permits Fubini's theorem.
Spatial integration by parts for $t>0$, followed by \eqref{eq:heat-weak-time}, gives
\begin{align}\label{eq:H-weak-proof}
 \langle\Div H,\phi\rangle
 &=-\frac1h\int_0^h t\int\ip{\nabla\phi}{\nabla\rho_t}\dx\dd t
 \notag\\
 &=\frac1h\int_0^h t\int\rho_t\Delta\phi\dx\dd t\notag\\
 &=\frac1h\int_0^h t\frac\dd{\dd t}
                    \left(\int\phi\rho_t\dx\right)\dd t\notag\\
 &=\frac1h\left[t\int\phi\rho_t\dx\right]_0^h
       -\frac1h\int_0^h\int\phi\rho_t\dx\dd t\notag\\
 &=\int\phi(\inv-\bar\rho)\dx.
\end{align}
The time integration by parts is valid because $t\mapsto\int\phi\rho_t\dx$ is absolutely continuous by \eqref{eq:heat-weak-integrated}. Its boundary term at zero vanishes since $\left|t\int\phi\rho_t\dx\right|\le t\norm\phi_\infty$.

\emph{The matrix flux.}
We prove the second identity in \eqref{eq:flux-identities}
componentwise.
For each fixed $i$, the difference quotients for the pairing in \eqref{eq:push-def} with density $(\partial_iU_\lambda)\inv$ are dominated by the integrable function $\norm{\nabla\phi}_\infty\norm{\drift}^2\inv$.
We may therefore differentiate under the integral.
Keeping the initial weight $\partial_iU_\lambda(x)$ fixed gives
\begin{align}\label{eq:vector-push-weak-time}
 &\frac\dd{\dd s}
   \int\phi(y)\bigl[(T_s)_\#(\drift\,\inv)\bigr]_i(y)\dd y
 \notag\\
 &\qquad=-\sum_j\int\partial_j\phi(T_sx)
       \partial_iU_\lambda(x)\partial_jU_\lambda(x)\inv(x)\dx
 \notag\\
 &\qquad=-\sum_j\int\partial_j\phi(y)
       \bigl[(T_s)_\#((\drift\otimes\drift)\inv)\bigr]_{ij}(y)\dd y.
\end{align}
The same second-moment bound justifies Fubini's theorem for the two time integrals below, since
\[
 \int_0^h\int_0^s\int\norm{\drift}^2\inv\dx\dd u\dd s
 =\frac{h^2}{2}\int\norm{\drift}^2\inv\dx<\infty.
\]
Integrating \eqref{eq:vector-push-weak-time} first from $0$ to $s$ and then averaging over $s\in[0,h]$, we obtain
\begin{align*}
 &\int\phi\bigl((\partial_iU_\lambda)\inv-J_i\bigr)\dx\\
 &\quad=-\frac1h\int_0^h\int_0^s\frac\dd{\dd u}
       \left(\int\phi(y)
           \bigl[(T_u)_\#(\drift\,\inv)\bigr]_i(y)\dd y\right)
       \dd u\dd s\\
 &\quad=\frac1h\sum_j\int_0^h(h-u)
       \int\partial_j\phi(y)
       \bigl[(T_u)_\#((\drift\otimes\drift)\inv)\bigr]_{ij}(y)
       \dd y\dd u\\
 &\quad=-\sum_j\int\partial_j\phi\,E_{ij}\dx
       =\langle(\Div E)_i,\phi\rangle.
\end{align*}
This proves $\drift\,\inv-J=\Div E$ and completes
\eqref{eq:flux-identities}.

\emph{The residual equation.}
We use the product rule
\begin{equation}\label{eq:flux-product-rule}
 \Div(\drift\otimes H)
 =\drift\,(\inv-\bar\rho)+(\hess)H.
\end{equation}
To justify it, test $\Div H=\inv-\bar\rho$ against smooth approximations of $(\partial_iU_\lambda)\phi$.
Mollifications of $\partial_iU_\lambda$ converge locally uniformly; their gradients are uniformly bounded and converge almost everywhere to its weak gradient.
Since $H$ and $\inv-\bar\rho$ are locally integrable, dominated convergence gives
\begin{align*}
 -\int(\partial_iU_\lambda)\ip{H}{\nabla\phi}\dx
 =\int(\partial_iU_\lambda)(\inv-\bar\rho)\phi\dx
 +\int\phi\sum_j(\partial_j\partial_iU_\lambda)H_j\dx,
\end{align*}
which is the $i$th component of \eqref{eq:flux-product-rule}.
Now $\gen^*\pi_\lambda=0$, \eqref{eq:flux-identities}, and \eqref{eq:flux-product-rule} give
\begin{align*}
 \gen^*(\bar\rho-\pi_\lambda)
 &=\Delta\bar\rho+\Div(\drift\,\bar\rho)\\
 &=\Div(\drift\,\bar\rho-J)\\
 &=\Div\bigl(\Div E-\drift\,(\inv-\bar\rho)\bigr)\\
 &=\Div\bigl(\Div E-\Div(\drift\otimes H)+(\hess)H\bigr)\\
 &=\Div\Div(E-\drift\otimes H)+\Div((\hess)H).
\end{align*}
All derivatives in this calculation are distributional.
Equivalently, the resulting identity means that
\[
 \int(\bar\rho-\pi_\lambda)\gen\phi\dx
 =\int\ip{E-\drift\otimes H}{\nabla^2\phi}\dx
  -\int\ip{(\hess)H}{\nabla\phi}\dx
\]
for every $\phi\in C_c^\infty(\R^d)$.
This proves \eqref{eq:exact-residual}.
\end{proof}

For the quantitative flux bounds below and in
\cref{sec:flux-proofs}, choose the moment order
\begin{equation}\label{eq:r-local}
 r:=4+2\log(1+L_\lambda/m).
\end{equation}

\begin{lemma}[Stationary flux bounds]\label{lem:flux}
Under \eqref{eq:bias-step}, the fields in \eqref{eq:J}--\eqref{eq:E} satisfy
\begin{align}
 \Lp{H/\pi_\lambda}{r}&\le ChrR\log\frac e{mh},\label{eq:H-bound}\\
 \Lp{E/\pi_\lambda}{2}&\le ChR^2,\label{eq:E-bound}\\
 \Lp{(E-\drift\otimes H)/\pi_\lambda}{2}
       &\le ChrR^2\log\frac e{mh}.\label{eq:A-bound}
\end{align}
\end{lemma}
The proof is given in \cref{sec:flux-proofs}.

\subsection{A residual-to-transport estimate}\label{sec:residual-lemma}
\begin{proposition}[Second-order residual estimate]\label{prop:residual}
Let $\nu,\rho\in\mathcal P_2(\R^d)$ be densities with
$\nu/\pi_\lambda,\rho/\pi_\lambda\in L^2(\pi_\lambda)$.
Suppose a vector density $H$ and a matrix density $E$ satisfy
\begin{align}
 \nu-\rho&=\Div H,\label{eq:abstract-H}\\
 \gen^*(\rho-\pi_\lambda)
 &=\Div\Div(E-\drift\otimes H)+\Div((\hess)H)\label{eq:abstract-residual}
\end{align}
in the sense of distributions. If the norms on the right below are finite, then
\begin{align}\label{eq:abstract-bound}
 \sqrt m\,W_2(\nu,\pi_\lambda)\le{}&
 2\sqrt m\Lp{H/\pi_\lambda}{2}
 +2\Lp{(E-\drift\otimes H)/\pi_\lambda}{2}\notag\\
 &+2\left(\int\frac{H^\top(\hess)H}{\pi_\lambda}\dx\right)^{1/2}.
\end{align}
\end{proposition}

For clarity, \eqref{eq:abstract-H} and \eqref{eq:abstract-residual} mean that, for every $v\in C_c^\infty(\R^d)$,
\begin{align*}
 \int v(\nu-\rho)\dx
 &=-\int\ip{\nabla v}{H}\dx,\\
 \int(\rho-\pi_\lambda)\gen v\dx
 &=\int\ip{E-\drift\otimes H}{\nabla^2v}\dx
   -\int\ip{(\hess)H}{\nabla v}\dx.
\end{align*}
All derivatives in these integral identities act on the test function.
In particular, no classical derivatives of $E$ or of the weak Hessian $\hess$ are required.

\begin{proof}[Proof of \cref{prop:residual}]
Fix $\phi\in C_c^\infty(\R^d)$ and let $\psi$ be its Poisson
solution from \cref{lem:poisson-energy}.
We first prove the pairing identity
\begin{align}\label{eq:weak-pairing}
 \int\phi(\rho-\pi_\lambda)\dx
 &=\int(\rho-\pi_\lambda)\gen\psi\dx\notag\\
 &=\int\ip{E-\drift\otimes H}{\nabla^2\psi}\dx
   -\int\ip{(\hess)H}{\nabla\psi}\dx.
\end{align}
The residual equation \eqref{eq:abstract-residual} is initially
available only for compactly supported smooth tests.
We therefore take the sequence $\psi_n\in C_c^\infty(\R^d)$
from \eqref{eq:poisson-approximation}.
For every $n$, the weak formulation gives
\begin{align}\label{eq:weak-pairing-approx}
 \int(\rho-\pi_\lambda)\gen\psi_n\dx
 =\int\ip{E-\drift\otimes H}{\nabla^2\psi_n}\dx
 -\int\ip{(\hess)H}{\nabla\psi_n}\dx.
\end{align}
We now justify passage to the limit in each integral of
\eqref{eq:weak-pairing-approx}.

For the curvature integral, positivity of $\hess$ and
Cauchy--Schwarz give, for every $v\in H^1(\pi_\lambda)$,
\begin{align}\label{eq:curvature-pairing-bound}
 \left|\int\ip{(\hess)H}{\nabla v}\dx\right|
 \le
 \left(\int\frac{H^\top(\hess)H}{\pi_\lambda}\dx\right)^{1/2}
 \left(\int\nabla v^\top(\hess)\nabla v
                         \dd\pi_\lambda\right)^{1/2}.
\end{align}
Using Cauchy--Schwarz for the density and matrix integrals,
and \eqref{eq:curvature-pairing-bound} with $v=\psi_n-\psi$
for the curvature integral, we obtain
\begin{align}\label{eq:pairing-limits}
 \left|\int(\rho-\pi_\lambda)
                 (\gen\psi_n-\gen\psi)\dx\right|
 &\le
 \Lp{\rho/\pi_\lambda-1}{2}
 \Lp{\gen\psi_n-\gen\psi}{2}
 \longrightarrow0,\notag\\[1mm]
 \left|\int\ip{E-\drift\otimes H}
                 {\nabla^2\psi_n-\nabla^2\psi}\dx\right|
 &\le
 \Lp{(E-\drift\otimes H)/\pi_\lambda}{2}
 \Lp{\nabla^2\psi_n-\nabla^2\psi}{2}
 \longrightarrow0,\notag\\[1mm]
 \left|\int\ip{(\hess)H}{\nabla\psi_n-\nabla\psi}\dx\right|
 &\le
 \sqrt{L_\lambda}
 \left(\int\frac{H^\top(\hess)H}{\pi_\lambda}\dx\right)^{1/2}
 \Lp{\nabla\psi_n-\nabla\psi}{2}
 \longrightarrow0.
\end{align}
The coefficient norms are finite by the hypotheses of the
proposition.
The three approximation errors tend to zero by
\eqref{eq:poisson-approximation}; the last estimate also uses
$\hess\preceq L_\lambda I$.
Thus \eqref{eq:pairing-limits} allows us to pass to the limit in
\eqref{eq:weak-pairing-approx}, proving \eqref{eq:weak-pairing}.

We next estimate the right-hand side of
\eqref{eq:weak-pairing}.
For the matrix integral, Cauchy--Schwarz uses the Frobenius
inner product; for the curvature integral, apply
\eqref{eq:curvature-pairing-bound} with $v=\psi$.
Consequently,
\begin{align}\label{eq:residual-energy-products}
 \left|\int\phi(\rho-\pi_\lambda)\dx\right|
 &\le
 \Lp{(E-\drift\otimes H)/\pi_\lambda}{2}
 \left(\int\norm{\nabla^2\psi}_F^2
                         \dd\pi_\lambda\right)^{1/2}\notag\\
 &\quad+
 \left(\int\frac{H^\top(\hess)H}{\pi_\lambda}\dx\right)^{1/2}
 \left(\int\nabla\psi^\top(\hess)\nabla\psi
                         \dd\pi_\lambda\right)^{1/2}.
\end{align}
From \eqref{eq:bochner-main},
\begin{equation}\label{eq:poisson-factor-bounds}
 \begin{split}
 \left(\int\norm{\nabla^2\psi}_F^2
                         \dd\pi_\lambda\right)^{1/2}
 &\le\frac1{\sqrt m}\Lp{\nabla\phi}{2},\\
 \left(\int\nabla\psi^\top(\hess)\nabla\psi
                         \dd\pi_\lambda\right)^{1/2}
 &\le\frac1{\sqrt m}\Lp{\nabla\phi}{2}.
 \end{split}
\end{equation}
Substituting \eqref{eq:poisson-factor-bounds} into
\eqref{eq:residual-energy-products} yields
\begin{align}\label{eq:rho-negative}
 \left|\int\phi(\rho-\pi_\lambda)\dx\right|
 \le\frac{\Lp{\nabla\phi}{2}}{\sqrt m}
 \left[
 \Lp{(E-\drift\otimes H)/\pi_\lambda}{2}
 +\left(\int\frac{H^\top(\hess)H}{\pi_\lambda}\dx\right)^{1/2}
 \right].
\end{align}

The first residual identity \eqref{eq:abstract-H}, tested
directly with $\phi$, gives
\begin{align}\label{eq:nu-rho-observable}
 \left|\int\phi(\nu-\rho)\dx\right|
 =\left|\int\ip{\nabla\phi}{H}\dx\right|
 \le\Lp{H/\pi_\lambda}{2}\Lp{\nabla\phi}{2}.
\end{align}
Since $\nu-\pi_\lambda=(\nu-\rho)+(\rho-\pi_\lambda)$,
the triangle inequality, \eqref{eq:nu-rho-observable}, and
\eqref{eq:rho-negative} imply
\begin{align}\label{eq:nu-observable-bound}
 \left|\int\phi(\nu-\pi_\lambda)\dx\right|
 &\le
 \left|\int\phi(\nu-\rho)\dx\right|
 +\left|\int\phi(\rho-\pi_\lambda)\dx\right|\notag\\
 &\le
 \Lp{\nabla\phi}{2}
 \biggl[
 \Lp{H/\pi_\lambda}{2}
 +\frac1{\sqrt m}
      \Lp{(E-\drift\otimes H)/\pi_\lambda}{2}\notag\\
 &\hspace{35mm}
 +\frac1{\sqrt m}
 \left(\int\frac{H^\top(\hess)H}{\pi_\lambda}\dx\right)^{1/2}
 \biggr].
\end{align}
Taking the supremum in \eqref{eq:nu-observable-bound} over the
test functions in \eqref{eq:Hminus-def} gives
\begin{align}\label{eq:nu-negative-bound}
 \Nminus{\nu-\pi_\lambda}
 &=
 \sup_{\substack{\phi\in C_c^\infty(\R^d)\\
                  \Lp{\nabla\phi}{2}\le1}}
 \left|\int\phi(\nu-\pi_\lambda)\dx\right|\notag\\
 &\le
 \Lp{H/\pi_\lambda}{2}
 +\frac1{\sqrt m}
      \Lp{(E-\drift\otimes H)/\pi_\lambda}{2}\notag\\
 &\quad+
 \frac1{\sqrt m}
 \left(\int\frac{H^\top(\hess)H}{\pi_\lambda}\dx\right)^{1/2}.
\end{align}
Finally, $\nu/\pi_\lambda\in L^2(\pi_\lambda)$ by assumption,
so \eqref{eq:transport-comparison} applies.
Using that comparison and then \eqref{eq:nu-negative-bound},
we conclude that
\begin{align*}
 \sqrt m\,W_2(\nu,\pi_\lambda)
 &\le2\sqrt m\,\Nminus{\nu-\pi_\lambda}\\
 &\le
 2\sqrt m\,\Lp{H/\pi_\lambda}{2}
 +2\Lp{(E-\drift\otimes H)/\pi_\lambda}{2}\\
 &\quad+
 2\left(\int\frac{H^\top(\hess)H}{\pi_\lambda}\dx\right)^{1/2}.
\end{align*}
This is \eqref{eq:abstract-bound}.
\end{proof}

\subsection{Proof of the fixed-smoothing theorem}\label{sec:bias-proof}
\begin{proof}[Proof of \cref{thm:bias}]
Assume \eqref{eq:bias-step}.
Since $\tau_f\ge md\ge m$ and $L_\lambda=L_f+\lambda^{-1}$, we have
\begin{equation}\label{eq:bias-basic-step}
 h(L_\lambda+R^2)
 \le h\left[\frac{L_f\tau_f}{m}+R^2+\frac1\lambda\right]
 \le\frac{c}{\ell_{h,\lambda}^2}
 \le c.
\end{equation}
Taking the universal constant in \eqref{eq:bias-step} sufficiently small gives $hL_\lambda\le1/2$.
By \cref{prop:invariant}, $Q_{\lambda,h}$ therefore has a unique invariant law $\inv\in\mathcal P_2(\R^d)$ with a strictly positive smooth density.

Let $r$ be as in \eqref{eq:r-local}, and take $\bar\rho,H,E$ from \eqref{eq:heat-interpolation}--\eqref{eq:E}. We will apply \cref{prop:residual} with $\nu=\inv$ and $\rho=\bar\rho$. We first verify its density assumptions and then estimate the three terms in \eqref{eq:abstract-bound}.

\emph{The density assumptions.}
\eqref{eq:bias-step} gives
\begin{align}\label{eq:bias-ratio-conditions}
 64h\frac{L_f\tau_f}{m}
   +8h\frac{L_fR}{\sqrt m}
   +64h\frac{GR}{m\lambda}
 \le
 64h\left[
       \frac{L_f\tau_f}{m}
       +\frac{L_fR}{\sqrt m}
       +\frac{GR}{m\lambda}
      \right]
 \le\frac{64c}{\ell_{h,\lambda}^2}.
\end{align}
Together with \eqref{eq:bias-basic-step}, these bounds imply the conditions of \cref{lem:ratios} with $\alpha=8$, after decreasing the universal constant in \eqref{eq:bias-step}. Evaluating \eqref{eq:ratios} at $t=0$ yields
\begin{equation}\label{eq:bias-invariant-ratio}
 \Lp{\inv/\pi_\lambda}{2}
 \le\Lp{\inv/\pi_\lambda}{8}
 \le e^{1/4}.
\end{equation}

For $0<t\le h$, apply \eqref{eq:op-zero} with output order $s=2$ and \eqref{eq:op-push} with output order $s=4$. Their time restrictions hold because
\[
 4tR^2\le4hR^2\le4c,
 \qquad
 4h(L_\lambda+R^2)\le4c
\]
by \eqref{eq:bias-basic-step}.
Thus, after the same choice of a sufficiently small constant,
\begin{align}\label{eq:bias-heat-ratio}
 \Lp{\rho_t/\pi_\lambda}{2}
 &=\Lp{\heat_t((T_h)_\#\inv)/\pi_\lambda}{2}\notag\\
 &\le C\Lp{(T_h)_\#\inv/\pi_\lambda}{4}\notag\\
 &\le C\Lp{\inv/\pi_\lambda}{8}
 \le C.
\end{align}
Jensen's inequality for the time average, followed by Tonelli's theorem and \eqref{eq:bias-heat-ratio}, gives
\begin{align}\label{eq:bias-average-ratio}
 \Lp{\bar\rho/\pi_\lambda}{2}^2
 &=\int
 \left(\frac1h\int_0^h
             \frac{\rho_t(x)}{\pi_\lambda(x)}\dd t\right)^2
 \dd\pi_\lambda(x)\notag\\
 &\le\frac1h\int_0^h
                 \Lp{\rho_t/\pi_\lambda}{2}^2\dd t
 <\infty.
\end{align}
The heat-kernel representation also gives
\[
 \int\norm y^2\bar\rho(y)\dd y
 =\int\norm{T_h(x)}^2\inv(x)\dx+dh<\infty.
\]
Indeed, the Gaussian increment at time $t$ contributes $2td$ to the second moment, and $T_h$ has at most linear growth. Hence $\bar\rho\in\mathcal P_2(\R^d)$.
Equations \eqref{eq:bias-invariant-ratio} and \eqref{eq:bias-average-ratio} verify the two $L^2$ density-ratio assumptions of \cref{prop:residual}.

\emph{The curvature term.}
The flux bounds \eqref{eq:H-bound} and \eqref{eq:A-bound} are supplied by \cref{lem:flux}.
It remains to estimate the integral involving $\hess$ in \eqref{eq:abstract-bound}.
Since $\hess$ is positive semidefinite, its quadratic form is bounded by its trace times the squared Euclidean norm.
Applying this pointwise and then using H\"older with conjugate exponents $r/2$ and $r/(r-2)$ gives
\begin{align}\label{eq:curvature-holder}
 \int\frac{H^\top(\hess)H}{\pi_\lambda}\dx
 &=\int
 \left(\frac H{\pi_\lambda}\right)^\top
 (\hess)
 \left(\frac H{\pi_\lambda}\right)\dd\pi_\lambda\notag\\
 &\le\int
       \norm{H/\pi_\lambda}^2\tr(\hess)\dd\pi_\lambda\notag\\
 &\le
 \Lp{H/\pi_\lambda}{r}^2
 \left(
   \int[\tr(\hess)]^{r/(r-2)}\dd\pi_\lambda
 \right)^{(r-2)/r}.
\end{align}
To estimate the last integral, use $0\le\tr(\hess)\le dL_\lambda$ almost everywhere and \eqref{eq:target-second}.
Since $r/(r-2)=1+2/(r-2)$,
\begin{align}\label{eq:curvature-trace-moment}
 \int[\tr(\hess)]^{r/(r-2)}\dd\pi_\lambda
 &=\int
 \tr(\hess)[\tr(\hess)]^{2/(r-2)}\dd\pi_\lambda\notag\\
 &\le(dL_\lambda)^{2/(r-2)}
       \int\tr(\hess)\dd\pi_\lambda\notag\\
 &\le(dL_\lambda)^{2/(r-2)}R^2.
\end{align}
Moreover, $R^2\ge md$ and the choice of $r$ in \eqref{eq:r-local} imply
\begin{align}\label{eq:curvature-factor}
 \left(\frac{dL_\lambda}{R^2}\right)^{1/r}
 \le\left(\frac{L_\lambda}{m}\right)^{1/r}
 =\exp\left\{
 \frac{\log(L_\lambda/m)}
      {4+2\log(1+L_\lambda/m)}
 \right\}
 \le\sqrt e.
\end{align}
Substituting \eqref{eq:curvature-trace-moment} into \eqref{eq:curvature-holder}, taking square roots, and using \eqref{eq:curvature-factor}, we obtain
\begin{align}\label{eq:curvature-bound}
 \left(\int\frac{H^\top(\hess)H}{\pi_\lambda}\dx\right)^{1/2}
 &\le
 R\left(\frac{dL_\lambda}{R^2}\right)^{1/r}
 \Lp{H/\pi_\lambda}{r}\notag\\
 &\le\sqrt e\,R\Lp{H/\pi_\lambda}{r}.
\end{align}
In particular, this integral is finite by \eqref{eq:H-bound}.

\emph{Application of the residual estimate.}
The identities \eqref{eq:flux-identities} and \eqref{eq:exact-residual} are precisely the weak identities required in \cref{prop:residual} for $\nu=\inv$ and $\rho=\bar\rho$. Its density assumptions have been verified above, and its weighted norms are finite by \eqref{eq:H-bound}, \eqref{eq:A-bound}, and \eqref{eq:curvature-bound}. Since $r\ge4$, $\Lp{H/\pi_\lambda}{2}\le\Lp{H/\pi_\lambda}{r}$. Applying \cref{prop:residual} therefore gives
\begin{align}\label{eq:bias-from-flux}
 \sqrt m\,W_2(\inv,\pi_\lambda)
 &\le
 2\sqrt m\,\Lp{H/\pi_\lambda}{2}
 +2\Lp{(E-\drift\otimes H)/\pi_\lambda}{2}\notag\\
 &\quad+
 2\left(\int\frac{H^\top(\hess)H}{\pi_\lambda}\dx\right)^{1/2}
 \notag\\
 &\le
 2(\sqrt m+\sqrt e\,R)\Lp{H/\pi_\lambda}{r}
 +2\Lp{(E-\drift\otimes H)/\pi_\lambda}{2}\notag\\
 &\le ChrR^2\log\frac e{mh}.
\end{align}
The last inequality uses \eqref{eq:H-bound}, \eqref{eq:A-bound}, and $\sqrt m\le R$.

Finally, the definitions \eqref{eq:r-local} and \eqref{eq:log-h} give
\begin{equation}\label{eq:bias-log-bounds}
 r\le4\ell_{h,\lambda},
 \qquad
 \log\frac e{mh}
 =1+\log\frac1{mh}
 \le\ell_{h,\lambda}.
\end{equation}
Combining \eqref{eq:bias-from-flux} and
\eqref{eq:bias-log-bounds} yields
\[
 \sqrt m\,W_2(\inv,\pi_\lambda)
 \le ChR^2\ell_{h,\lambda}^2.
\]
This proves \eqref{eq:bias}.
\end{proof}

\subsection{Proof of the end-to-end complexity theorem}\label{sec:complexity-proof}
\begin{proof}[Proof of \cref{thm:complexity}]
We allocate $\varepsilon/4$ to the Moreau bias, $\varepsilon/4$ to the invariant-measure bias, and $\varepsilon/2$ to convergence from the initial law.
The choice \eqref{eq:lambda-eps} gives the first budget by \eqref{eq:moreau-bias}.
We verify that \eqref{eq:h-eps} gives the second and that \eqref{eq:N-choice} gives the third.
Substitute $\lambda=\varepsilon/G^2$ into the bracket in \eqref{eq:bias-step}.
Because $R\ge G$ and $\varepsilon\le1$,
\begin{align}\label{eq:step-comparison}
 \frac{L_f\tau_f}{m}+\frac{L_fR}{\sqrt m}
 +\frac{GR}{m\lambda}+R^2+\frac1\lambda
 \le 2\left[\frac{L_f\tau_f}{m}+\frac{L_fR}{\sqrt m}
            +\frac{R^2+G^3R/m}{\varepsilon}\right].
\end{align}
Also $R/\sqrt m\le G/\sqrt m+\sqrt{\tau_f/m}$.
The reciprocal of $mh_\varepsilon$ in \eqref{eq:h-eps} is consequently bounded above by a fixed polynomial in the dimensionless ratios appearing in \eqref{eq:log-eps}, multiplied by $\ell_\varepsilon^2/c$.
It follows that
\begin{equation}\label{eq:logs-compare}
 \ell_{h_\varepsilon,\lambda_\varepsilon}
 \le C\ell_\varepsilon+\log(1/c).
\end{equation}
Choose the absolute $c$ small enough, using $c(1+\log(1/c))^2\to0$. Equations \eqref{eq:step-comparison} and \eqref{eq:logs-compare} show that \eqref{eq:bias-step} holds and that
\begin{equation}\label{eq:discrete-budget}
 \sqrt m\,W_2(\widehat\pi_{\lambda_\varepsilon,h_\varepsilon},
                    \pi_{\lambda_\varepsilon})\le\varepsilon/4.
\end{equation}
Equation \eqref{eq:moreau-bias} gives another $\varepsilon/4$.

Synchronous coupling of \eqref{eq:MYULA} gives
\begin{equation}\label{eq:contraction}
 W_2(\mu_0Q_{\lambda,h}^N,\inv)
 \le(1-mh)^NW_2(\mu_0,\inv).
\end{equation}
The triangle inequality and the two bias budgets imply
\[
 \sqrt m\,W_2(\mu_0,\widehat\pi_{\lambda_\varepsilon,h_\varepsilon})
 \le\sqrt m\,W_2(\mu_0,\pi)+\varepsilon/2.
\]
Thus \eqref{eq:N-choice} makes the contraction error at most $\varepsilon/2$.
The triangle inequality gives
\begin{align*}
 \sqrt m\,W_2(\mu_0Q_{\lambda_\varepsilon,h_\varepsilon}^N,\pi)
 \le{}&\sqrt m\,W_2(\mu_0Q_{\lambda_\varepsilon,h_\varepsilon}^N,
                     \widehat\pi_{\lambda_\varepsilon,h_\varepsilon})\\
 &+\sqrt m\,W_2(\widehat\pi_{\lambda_\varepsilon,h_\varepsilon},
                         \pi_{\lambda_\varepsilon})
 +\sqrt m\,W_2(\pi_{\lambda_\varepsilon},\pi)
 \le\varepsilon.
\end{align*}
This proves \eqref{eq:accuracy}; substituting \eqref{eq:h-eps} into \eqref{eq:N-choice} proves \eqref{eq:N-explicit}.

This completes the proof of \cref{thm:complexity}.
\end{proof}

\section{Stationary Flux Estimates}\label{sec:flux-proofs}
This section proves the bounds on $H$ and $E$ in \cref{lem:flux}; the fields are defined in \cref{sec:flux}.
We first bound the averaged drift $J$, then use a finite heat expansion to estimate the gradient of the invariant density $\inv$ and obtain the bound on $H$.
The bound on $E$ follows from the drift moment and pushforward estimates.

\subsection{The averaged drift flux}\label{sec:J-bound}
Assume \eqref{eq:bias-step}.
Since $\tau_f\ge m$, this condition implies
$hL_\lambda\le c$, so \cref{prop:invariant} supplies the
invariant density $\inv$.

By \eqref{eq:r-local} and \eqref{eq:log-h}, $r\le4\ell_{h,\lambda}$, and hence \eqref{eq:bias-step} implies
\begin{align*}
 h(32r)^2
   \left[\frac{L_f\tau_f}{m}+\frac{GR}{m\lambda}\right]
   +h(32r)\frac{L_fR}{\sqrt m}
 \le
 C h\ell_{h,\lambda}^2
   \left[
     \frac{L_f\tau_f}{m}
     +\frac{L_fR}{\sqrt m}
     +\frac{GR}{m\lambda}
   \right]
 \le Cc.
\end{align*}
Taking the universal constant in \eqref{eq:bias-step}
sufficiently small therefore ensures \eqref{eq:ratio-step}
with $\alpha=32r$.
By \eqref{eq:ratios} with $\alpha=32r$ and $t=0$,
\begin{equation}\label{eq:stationary-ratio-orders}
 \Lp{\inv/\pi_\lambda}{32r}
 \le e^{1/4}.
\end{equation}
For the next estimate, we use \eqref{eq:op-push} with output
order $s=8r$ and times $0\le t\le h$.
Its restriction is satisfied after decreasing the constant in
\eqref{eq:bias-step}, since
$L_\lambda\le L_f\tau_f/m+\lambda^{-1}$ and
$8rh(L_\lambda+R^2)\le Cc$.
Minkowski's inequality for the time average in \eqref{eq:J},
followed by \eqref{eq:op-push} and H\"older's inequality, gives
\begin{align}\label{eq:J-bound}
 \Lp{J/\pi_\lambda}{8r}
 &\le\frac1h\int_0^h
       \Lp{(T_u)_\#(\drift\,\inv)/\pi_\lambda}{8r}\dd u
       \notag\\
 &\le C\Lp{\drift\,\inv/\pi_\lambda}{16r}\notag\\
 &\le C\Lp{\drift}{32r}
         \Lp{\inv/\pi_\lambda}{32r}
 \le C\sqrt r\,R.
\end{align}
The last step uses \eqref{eq:target-high} at order $32r$
and \eqref{eq:stationary-ratio-orders}.

\subsection{A finite heat expansion}\label{sec:gradient}
We next derive the finite heat expansion.
By \eqref{eq:drift-balance}, the distributional divergence of
$J$ has the $L^1$ representative
\[
 \Div J=\frac{(T_h)_\#\inv-\inv}{h}.
\]
Since $J\in L^1(\R^d;\R^d)$, both $J$ and $\Div J$ can
therefore be convolved with the heat kernel.
For every $t>0$,
\begin{equation}\label{eq:heat-div-commutation}
 \heat_t(\Div J)=\Div(\heat_tJ).
\end{equation}
To verify this identity, take $\phi\in C_c^\infty(\R^d)$.
Symmetry of the heat kernel and
$\nabla\heat_t\phi=\heat_t\nabla\phi$ give
\begin{align*}
 \int\phi\,\heat_t(\Div J)\dx
 &=\int(\heat_t\phi)\Div J\dx\\
 &=-\int\ip{\nabla\heat_t\phi}{J}\dx\\
 &=-\int\ip{\nabla\phi}{\heat_tJ}\dx
 =\int\phi\,\Div(\heat_tJ)\dx.
\end{align*}
For the integration by parts with $\heat_t\phi$, first multiply
it by smooth cutoffs $\chi_k$ equal to one on the ball of
radius $k$, supported in the ball of radius $2k$, and satisfying
$\norm{\nabla\chi_k}_\infty\le C/k$.
The cutoff error is bounded by
\[
 \norm{\heat_t\phi}_\infty
 \norm{\nabla\chi_k}_\infty\norm J_{L^1}
 \longrightarrow0.
\]
The remaining terms converge by dominated convergence, since
$J,\Div J\in L^1$ and $\heat_t\phi,\nabla\heat_t\phi$ are bounded.
This proves \eqref{eq:heat-div-commutation}.

Applying $\heat_h$ to \eqref{eq:drift-balance}, and using
stationarity \eqref{eq:stationarity} and
\eqref{eq:heat-div-commutation}, yields
\begin{align}\label{eq:stationary-heat-step}
 \inv
 &=\heat_h((T_h)_\#\inv)\notag\\
 &=\heat_h\inv+h\heat_h(\Div J)\notag\\
 &=\heat_h\inv+h\Div\heat_hJ.
\end{align}
For $k=0,\ldots,n-1$, the semigroup property and
\eqref{eq:heat-div-commutation} then give
\[
 \heat_{kh}\inv-\heat_{(k+1)h}\inv
 =h\Div\heat_{(k+1)h}J.
\]
Summing these equalities gives the exact finite identity
\begin{equation}\label{eq:finite-heat}
 \inv=\heat_{nh}\inv
       +h\sum_{j=1}^n\Div\heat_{jh}J,
 \qquad n\ge1.
\end{equation}
For positive time, convolution of an $L^1$ function with the
Gaussian kernel is smooth: every spatial derivative can be
placed on the kernel, whose derivatives are bounded and
integrable.
Thus the right-hand side of \eqref{eq:finite-heat} is smooth.
The identity, initially obtained distributionally, consequently
holds pointwise for the smooth representative of $\inv$.

Choose
\[
 \tau:=\frac{c_1}{r^2R^2},
 \qquad n:=\left\lfloor\frac{\tau}{h}\right\rfloor,
\]
where $0<c_1\le1$ is a sufficiently small universal constant.
Since $r\le4\ell_{h,\lambda}$, \eqref{eq:bias-step} gives
$hr^2R^2\le16c$.
Decreasing the constant $c$ in \eqref{eq:bias-step} relative
to $c_1$ therefore ensures
\[
 h\le\frac{\tau}{2},
 \qquad \frac{\tau}{2}\le nh\le\tau.
\]
We will apply \eqref{eq:op-gradient} and \eqref{eq:op-second}
with output order $s=2r$, at times $t=nh$ and $t=jh$,
respectively.
Their common time restriction holds because, for $1\le j\le n$,
\begin{equation}\label{eq:heat-expansion-time-bound}
 (jh)(2r)^2R^2\le\tau(2r)^2R^2=4c_1.
\end{equation}

All terms in \eqref{eq:finite-heat} are smooth, and its sum is
finite.
We may therefore differentiate it to obtain
\begin{equation}\label{eq:finite-heat-gradient}
 \nabla\inv
 =\nabla\heat_{nh}\inv
  +h\sum_{j=1}^n\nabla\Div\heat_{jh}J.
\end{equation}

Taking the weighted $L^{2r}$ norm in
\eqref{eq:finite-heat-gradient}, and using
\eqref{eq:op-gradient}, \eqref{eq:op-second}, and
$nh\ge\tau/2$, gives
\begin{align}\label{eq:density-gradient}
 \Lp{\nabla\inv/\pi_\lambda}{2r}
 &\le
 \Lp{\nabla\heat_{nh}\inv/\pi_\lambda}{2r}
 +h\sum_{j=1}^n
       \Lp{\nabla\Div\heat_{jh}J/\pi_\lambda}{2r}
       \notag\\
 &\le C\tau^{-1/2}\Lp{\inv/\pi_\lambda}{4r}
       +Ch\sum_{j=1}^n
         \frac{\sqrt r}{jh}\Lp{J/\pi_\lambda}{8r}
       \notag\\
 &\le CrR\left(1+\sum_{j=1}^n\frac1j\right)
 \le CrR(1+\log n)
 \le CrR\log\frac e{mh}.
\end{align}
For the third inequality in \eqref{eq:density-gradient},
use \eqref{eq:J-bound},
$\tau^{-1/2}=rR/\sqrt{c_1}$, and
\eqref{eq:stationary-ratio-orders}.
The last two inequalities use
$\sum_{j=1}^n j^{-1}\le1+\log n$ and
$n\le\tau/h\le1/(mh)$, since $R^2\ge m$ and $c_1\le1$.

\subsection{Proof of the flux bounds}\label{sec:flux-bounds-proof}
\begin{proof}[Proof of \cref{lem:flux}]
To estimate $H$, we express $\nabla\rho_t$ in terms of
$\nabla\inv$ and $J$.
We first justify the commutation identity
\begin{equation}\label{eq:heat-gradient-commutation}
 \nabla(\heat_t\inv)=\heat_t(\nabla\inv),
 \qquad t>0.
\end{equation}
By \eqref{eq:density-gradient} and H\"older's inequality,
\[
 \int\norm{\nabla\inv}\dx
 =\int\norm{\nabla\inv/\pi_\lambda}\dd\pi_\lambda
 \le\Lp{\nabla\inv/\pi_\lambda}{2r}<\infty.
\]
Thus both $\inv$ and $\nabla\inv$ belong to $L^1(\R^d)$.
For each $i$, differentiation of the Gaussian kernel and
integration by parts give
\begin{align*}
 \partial_i(\heat_t\inv)(x)
 &=\int
   \partial_{x_i}
   \left[
    \frac{e^{-\norm{x-y}^2/(4t)}}{(4\pi t)^{d/2}}
   \right]\inv(y)\dd y\\
 &=-\int
   \partial_{y_i}
   \left[
    \frac{e^{-\norm{x-y}^2/(4t)}}{(4\pi t)^{d/2}}
   \right]\inv(y)\dd y\\
 &=\int
   \frac{e^{-\norm{x-y}^2/(4t)}}{(4\pi t)^{d/2}}
   \partial_i\inv(y)\dd y
 =\heat_t(\partial_i\inv)(x).
\end{align*}
For fixed $t>0$, the kernel and its first derivatives are
bounded.
Differentiation under the integral therefore follows from
$\inv\in L^1$.
The integration by parts is justified by the cutoff argument
used for \eqref{eq:heat-div-commutation}, now using
$\inv,\nabla\inv\in L^1$.
This proves \eqref{eq:heat-gradient-commutation}.

For $0<t\le h$, the definition of $\rho_t$,
\eqref{eq:drift-balance}, and
\eqref{eq:heat-div-commutation} give
\[
 \rho_t
 =\heat_t\bigl(\inv+h\Div J\bigr)
 =\heat_t\inv+h\Div\heat_tJ.
\]
Differentiating this identity between smooth functions and
using \eqref{eq:heat-gradient-commutation}, we obtain
\begin{align}\label{eq:heat-path-gradient}
 \nabla\rho_t
 &=\nabla\heat_t\inv+h\nabla\Div\heat_tJ\notag\\
 &=\heat_t(\nabla\inv)+h\nabla\Div\heat_tJ,
 \qquad 0<t\le h.
\end{align}
Apply \eqref{eq:op-zero} to the vector field $\nabla\inv$
and \eqref{eq:op-second} to $J$, both with output order $s=r$.
The time restrictions hold because
$tr^2R^2\le hr^2R^2\le c_1/2$, as verified in
\cref{sec:gradient}.
Using \eqref{eq:J-bound}, \eqref{eq:density-gradient}, and
$\Lp{J/\pi_\lambda}{4r}\le\Lp{J/\pi_\lambda}{8r}$, we get
\begin{align*}
 \Lp{\nabla\rho_t/\pi_\lambda}{r}
 &\le
 C\Lp{\nabla\inv/\pi_\lambda}{2r}
 +\frac{Ch\sqrt r}{t}\Lp{J/\pi_\lambda}{4r}\\
 &\le CrR\log\frac e{mh}+\frac{ChrR}{t},
 \qquad 0<t\le h.
\end{align*}
The weight $t$ in \eqref{eq:H} cancels the apparent singularity at zero, proving \eqref{eq:H-bound}.
For \eqref{eq:E-bound}, apply \eqref{eq:op-push} with output
order $s=2$ and time $u\in[0,h]$.
Its time restriction is weaker than the one already verified
with output order $8r$ before \eqref{eq:J-bound}.
Using \eqref{eq:outer} and H\"older's inequality, we obtain
\begin{align}\label{eq:matrix-push-bound}
 &\Lp{(T_u)_\#((\drift\otimes\drift)\inv)/\pi_\lambda}{2}
 \notag\\
 &\qquad\le
 C\Lp{(\drift\otimes\drift)\inv/\pi_\lambda}{4}\notag\\
 &\qquad\le
 C\Lp{\drift}{16}^2\Lp{\inv/\pi_\lambda}{8}
 \le CR^2,
 \qquad 0\le u\le h.
\end{align}
The last inequality follows from \eqref{eq:target-high} at order $16$ and \eqref{eq:stationary-ratio-orders}.
Minkowski's inequality in \eqref{eq:E}, followed by
\eqref{eq:matrix-push-bound}, now gives
\begin{align*}
 \Lp{E/\pi_\lambda}{2}
 &\le\frac1h\int_0^h(h-u)
   \Lp{(T_u)_\#((\drift\otimes\drift)\inv)/\pi_\lambda}{2}
   \dd u\\
 &\le\frac{CR^2}{h}\int_0^h(h-u)\dd u
 \le ChR^2.
\end{align*}
This proves \eqref{eq:E-bound}.
Finally, $r\ge4$ and \eqref{eq:target-high} imply
\[
 \Lp{\drift\otimes H/\pi_\lambda}{2}
 \le\Lp{\drift}{4}\Lp{H/\pi_\lambda}{4}
 \le CR\Lp{H/\pi_\lambda}{r}.
\]
Combining this with \eqref{eq:E-bound} proves \eqref{eq:A-bound}.
This proves \eqref{eq:H-bound}, \eqref{eq:E-bound}, and \eqref{eq:A-bound}, completing the proof of \cref{lem:flux}.
\end{proof}

\section{Conclusion}\label{sec:conclusion}
We established a $\widetilde O(h)$ invariant-measure bias bound
relative to the smoothed target under an explicit step-size condition,
with only logarithmic dependence on $\lambda^{-1}$ in the error
coefficient. Combining this bound with Moreau approximation and
Wasserstein contraction gives $\widetilde O(\varepsilon^{-1})$
iterations to achieve $\sqrt m\,W_2(\mu_N,\pi)\le\varepsilon$
for fixed model parameters and initialization. These guarantees
require no additional higher-order smoothness assumptions and
leave the classical MYULA algorithm unchanged.

\appendix
\section{Target Moments and Normalized Operator Estimates}\label[appendix]{app:operators}
\subsection{Target drift moments}\label{app:moments}
\begin{proof}[Proof of \cref{lem:moments}]
The identity and bound in \eqref{eq:target-second} follow from \citet[Lemma~7.1 and Proposition~A.4]{XinZhang}.
We prove the higher-moment and exponential estimates.

Since $\pi_\lambda$ has Gaussian tails and $\nabla f$ has at most linear growth, all moments of $\nabla f$ are finite.
Integration by parts with polynomially growing vector fields is justified by spatial cutoffs; see \cref{app:approximation}.

For $s\ge2$, apply integration by parts to the $C^1$ test field $\norm{\nabla f}^{s-2}\nabla f$.
Using $\drift=\nabla f+\nabla g_\lambda$, we obtain
\begin{align}\label{eq:moment-ibp}
 \int\norm{\nabla f}^s\dd\pi_\lambda
 &=\int\norm{\nabla f}^{s-2}\Delta f\dd\pi_\lambda
   \notag\\
 &\quad +(s-2)\int\norm{\nabla f}^{s-4}
       (\nabla f)^\top\nabla^2f\,\nabla f\dd\pi_\lambda
   \notag\\
 &\quad -\int\norm{\nabla f}^{s-2}
       \ip{\nabla f}{\nabla g_\lambda}\dd\pi_\lambda
   \notag\\
 &\le [\tau_f+(s-2)\min(L_f,\tau_f)]
       \int\norm{\nabla f}^{s-2}\dd\pi_\lambda
   \notag\\
 &\quad +G\int\norm{\nabla f}^{s-1}\dd\pi_\lambda.
\end{align}
The inequality uses $\Delta f\le\tau_f$, $0\preceq\nabla^2f\preceq\min(L_f,\tau_f)I$, and $\norm{\nabla g_\lambda}\le G$.
The term with coefficient $s-2$ is zero when $s=2$; for $s>2$, its integrand extends continuously by zero at $\nabla f=0$.

H\"older's inequality gives
\[
 \int\norm{\nabla f}^{s-2}\dd\pi_\lambda
 \le \Lp{\nabla f}{s}^{s-2},
 \qquad
 \int\norm{\nabla f}^{s-1}\dd\pi_\lambda
 \le \Lp{\nabla f}{s}^{s-1}.
\]
If $\Lp{\nabla f}{s}>0$, substituting these bounds into \eqref{eq:moment-ibp} and dividing by $\Lp{\nabla f}{s}^{s-2}$ yields
\[
 \Lp{\nabla f}{s}^2
 \le \tau_f+(s-2)\min(L_f,\tau_f)
      +G\Lp{\nabla f}{s}.
\]
Solving this quadratic inequality gives
\begin{equation}\label{eq:v-moment}
 \Lp{\nabla f}{s}
 \le G+\sqrt{\tau_f+(s-2)\min(L_f,\tau_f)}.
\end{equation}
Since $\norm{\nabla g_\lambda}\le G$, $G\le R$, and
$\tau_f\le R^2$, it follows that
\[
 \Lp{\drift}{s}
 \le \Lp{\nabla f}{s}+G
 \le 2G+\sqrt{(s-1)\tau_f}
 \le 3\sqrt{s}\,R,
\]
which proves \eqref{eq:target-high}.

Finally, \eqref{eq:target-high} gives $\int\norm{\drift}^{2k}\dd\pi_\lambda\le(18kR^2)^k$. Expanding the exponential and using $k!\ge(k/e)^k$, we obtain
\[
 \int e^{\norm{\drift}^2/(256R^2)}\dd\pi_\lambda
 \le 1+\sum_{k\ge1}(18e/256)^k
 \le 2.
\]
This proves \eqref{eq:target-exp}.
\end{proof}

\subsection{A Gibbs change of variables for the heat operator}\label{app:heat}
This and the following two subsections prove \cref{lem:operators}.
We use the exponential moment bound \eqref{eq:target-exp}:
\begin{equation}\label{eq:generic-exp}
 \int e^{\norm{\drift}^2/(256R^2)}\dd\pi_\lambda\le2.
\end{equation}
Let $X\sim\pi_\lambda$ and $Z_t\sim N(0,2tI_d)$ be independent, and write $\varphi_{2tI_d}$ for the density of $Z_t$.
The Gibbs density satisfies
\[
 \pi_\lambda(x)
 =\pi_\lambda(y)e^{U_\lambda(y)-U_\lambda(x)}.
\]
For every $a\ge0$, independence and the substitution $y=x+z$ therefore give
\begin{align}\label{eq:Gibbs-swap}
 &\E e^{a[U_\lambda(X+Z_t)-U_\lambda(X)]}\notag\\
 &=\iint\pi_\lambda(x)\varphi_{2tI_d}(y-x)
       e^{a[U_\lambda(y)-U_\lambda(x)]}\dx\dd y
       \notag\\
 &=\iint\pi_\lambda(y)\varphi_{2tI_d}(y-x)
       e^{(a+1)[U_\lambda(y)-U_\lambda(x)]}\dx\dd y
       \notag\\
 &\le\int e^{(a+1)^2t\norm{\drift(y)}^2}
       \pi_\lambda(y)\dd y.
\end{align}
For the last inequality, convexity gives
\[
 U_\lambda(y)-U_\lambda(x)
 \le\ip{\drift(y)}{y-x},
\]
and the Gaussian integral in $x$ is
\[
 \int\varphi_{2tI_d}(y-x)
      e^{(a+1)\ip{\drift(y)}{y-x}}\dx
 =e^{(a+1)^2t\norm{\drift(y)}^2}.
\]
Thus, by \eqref{eq:generic-exp}, the expectation in \eqref{eq:Gibbs-swap} is at most $2$ whenever $(a+1)^2tR^2\le1/256$.

Write $w=F/\pi_\lambda$.
Since the heat kernel has total mass one, Jensen's inequality gives
\[
 \norm{\heat_tF(y)}^s
 =\left\lVert
    \int\varphi_{2tI_d}(y-x)F(x)\dx
   \right\rVert^s
 \le\int\varphi_{2tI_d}(y-x)\norm{F(x)}^s\dx.
\]
Multiply by $\pi_\lambda(y)^{1-s}$ and integrate in $y$.
Using Tonelli's theorem and $F(x)=\pi_\lambda(x)w(x)$, we obtain
\begin{align}\label{eq:heat-zero-proof}
 \Lp{\heat_tF/\pi_\lambda}{s}^s
 &=\int\norm{\heat_tF(y)}^s
          \pi_\lambda(y)^{1-s}\dd y
          \notag\\
 &\le\iint\varphi_{2tI_d}(y-x)
          \norm{F(x)}^s\pi_\lambda(y)^{1-s}\dx\dd y
          \notag\\
 &=\iint\pi_\lambda(x)\varphi_{2tI_d}(y-x)
          \norm{w(x)}^s
          \left(\frac{\pi_\lambda(x)}{\pi_\lambda(y)}
          \right)^{s-1}\dx\dd y
          \notag\\
 &=\E\left[
       \norm{w(X)}^s
       e^{(s-1)[U_\lambda(X+Z_t)-U_\lambda(X)]}
      \right].
\end{align}
Cauchy--Schwarz then yields
\[
 \Lp{\heat_tF/\pi_\lambda}{s}^s
 \le\Lp{w}{2s}^s
    \left(
      \E e^{2(s-1)[U_\lambda(X+Z_t)-U_\lambda(X)]}
    \right)^{1/2}.
\]
Apply \eqref{eq:Gibbs-swap} with $a=2(s-1)$. Together with \eqref{eq:generic-exp}, this bounds the expectation by $2$ whenever $(2s-1)^2tR^2\le1/256$. Since $2s-1\le2s$, this condition follows from $ts^2R^2\le c$ for a sufficiently small universal constant $c$. Taking $s$th roots proves \eqref{eq:op-zero}.
The argument applies to scalar, vector, and matrix densities, using the absolute value, Euclidean norm, and Frobenius norm, respectively.

\subsection{Gaussian derivatives without an extra dimension factor}\label{app:heat-derivatives}
Let $F$ be scalar with $\Lp{F/\pi_\lambda}{2s}<\infty$.
Since $\pi_\lambda$ is a bounded probability density,
\[
 \int |F|\dx
 \le \Lp{F/\pi_\lambda}{2s},
 \qquad
 \int |F|^2\dx
 \le \norm{\pi_\lambda}_\infty
      \Lp{F/\pi_\lambda}{2s}^2.
\]
For fixed $t>0$, every derivative of the Gaussian kernel is bounded. Thus differentiation under the integral is justified by $F\in L^1(\R^d)$, and $\heat_tF$ is smooth.

For a unit vector $e$, differentiating the kernel gives
\begin{align*}
 e\cdot\nabla\heat_tF(x)
 &=\int e\cdot\nabla_x\varphi_{2tI_d}(x-y)F(y)\dd y\\
 &=-\int\frac{e\cdot(x-y)}{2t}
       \varphi_{2tI_d}(x-y)F(y)\dd y\\
 &=-\E\left[\frac{e\cdot Z_t}{2t}F(x-Z_t)\right].
\end{align*}
Using $\E(e\cdot Z_t)^2=2t$, Cauchy--Schwarz yields
\begin{align*}
 \norm{\nabla\heat_tF(x)}
 &=\sup_{\norm e=1}
   \left|\E\left[
     \frac{e\cdot Z_t}{2t}F(x-Z_t)
   \right]\right|\\
 &\le (2t)^{-1/2}
       \bigl(\heat_t|F|^2(x)\bigr)^{1/2}.
\end{align*}
For $s\ge2$, Jensen then gives $\norm{\nabla\heat_tF}^s\le(2t)^{-s/2}\heat_t|F|^s$. Multiply by $\pi_\lambda^{1-s}$, integrate, and use the same endpoint exchange as in \eqref{eq:heat-zero-proof}. This proves the scalar gradient estimate \eqref{eq:op-gradient}.

To estimate the divergence of a vector field, let $s'=s/(s-1)$ and $a=2s/(2s-1)$ within this calculation.
Initially take $\phi\in C_c^\infty(\R^d)$ with $\Lp{\phi}{s'}\le1$. Integration by parts against the compactly supported $\phi$ and symmetry of the Gaussian kernel give
\begin{align*}
 \int\phi\,\Div\heat_tF\dx
 &=-\int\ip{\nabla\phi}{\heat_tF}\dx\\
 &=-\int\ip{\heat_t(\nabla\phi)}{F}\dx\\
 &=-\int\ip{\nabla\heat_t\phi}{F/\pi_\lambda}
             \dd\pi_\lambda.
\end{align*}
The exchange of integrals is justified by $F\in L^1$ and the boundedness of $\nabla\phi$.
The last equality also uses $\heat_t(\nabla\phi)=\nabla\heat_t\phi$.
Since $1/(2s)+1/a=1$, H\"older's inequality with respect to $\pi_\lambda$ yields
\begin{equation}\label{eq:div-dual}
 \left|\int\phi\,\Div\heat_tF\dx\right|
 \le\Lp{F/\pi_\lambda}{2s}
       \Lp{\nabla\heat_t\phi}{a}.
\end{equation}
For every unit vector $e$, the scalar $e\cdot Z_t$ has distribution $N(0,2t)$, and hence
\[
 \left(\E|e\cdot Z_t|^{2s}\right)^{1/(2s)}
 \le C\sqrt{st}.
\]
Since $1/(2s)+1/a=1$, the derivative formula above and H\"older's inequality give
\begin{align*}
 \norm{\nabla\heat_t\phi(x)}
 &=\sup_{\norm e=1}
   \left|\E\left[
     \frac{e\cdot Z_t}{2t}\phi(x-Z_t)
   \right]\right|\\
 &\le \frac1{2t}\sup_{\norm e=1}
   \left(\E|e\cdot Z_t|^{2s}\right)^{1/(2s)}
   \left(\E|\phi(x-Z_t)|^a\right)^{1/a}\\
 &\le C\sqrt{s/t}\,
       \bigl(\heat_t|\phi|^a(x)\bigr)^{1/a}.
\end{align*}

To integrate this bound, use symmetry of the heat kernel and H\"older with conjugate exponents $s'/a=(2s-1)/(2s-2)$ and $2s-1$:
\begin{align*}
 \int\heat_t|\phi|^a\dd\pi_\lambda
 &=\int|\phi|^a
        \frac{\heat_t\pi_\lambda}{\pi_\lambda}
        \dd\pi_\lambda\\
 &\le \Lp{\phi}{s'}^a
       \Lp{\heat_t\pi_\lambda/\pi_\lambda}{2s-1}.
\end{align*}
Consequently,
\[
 \Lp{\nabla\heat_t\phi}{a}
 \le C\sqrt{s/t}\,\Lp{\phi}{s'}
       \Lp{\heat_t\pi_\lambda/\pi_\lambda}{2s-1}^{1/a}
 \le C\sqrt{s/t}.
\]
The last inequality uses $\Lp{\phi}{s'}\le1$ and \eqref{eq:op-zero} with input $\pi_\lambda$ and order $2s-1$.
Since $2s-1\le2s$, its time restriction follows from $ts^2R^2\le c$ after decreasing the universal constant $c$.
Duality and density of $C_c^\infty(\R^d)$ in $L^{s'}(\pi_\lambda)$ therefore yield
\begin{equation}\label{eq:div-heat}
 \Lp{\Div\heat_tF/\pi_\lambda}{s}
 \le C\sqrt{s/t}\Lp{F/\pi_\lambda}{2s}.
\end{equation}
Finally, $\nabla\Div\heat_t=\nabla\heat_{t/2}\circ\Div\heat_{t/2}$. Apply the scalar gradient bound and then \eqref{eq:div-heat}, with order $2s$ in the latter. The result is \eqref{eq:op-second}.

\subsection{Pushforward estimates by an integrated Jacobian}\label{app:push}
Recall that $0\preceq\hess\preceq L_\lambda I$ almost everywhere.
For $u\ge0$, the expansion map $S_u(x)=x+u\drift(x)$ is globally bi-Lipschitz.
Indeed, $S_u^{-1}y$ is the unique minimizer of the strongly convex function
$x\mapsto\frac12\norm{x-y}^2+uU_\lambda(x)$; monotonicity of $\drift$ gives a Lipschitz inverse.
For $u>0$, set $y=S_u(x)=x+u\drift(x)$.
Monotonicity of $\drift$ gives
\[
 \norm{\drift(x)}^2
 \le\ip{\drift(y)}{\drift(x)}
 \le\norm{\drift(y)}\norm{\drift(x)},
\]
so $\norm{\drift(x)}\le\norm{\drift(y)}$.
Convexity therefore implies
\[
 U_\lambda(y)-U_\lambda(x)
 \le\ip{\drift(y)}{y-x}
 =u\ip{\drift(y)}{\drift(x)}
 \le u\norm{\drift(y)}^2.
\]
Changing variables $y=S_u(x)$ and using the Gibbs density ratio,
we obtain
\begin{align}\label{eq:expanding-map}
 \int\det(I+u\hess(x))\pi_\lambda(x)\dx
 &=\int\pi_\lambda(S_u^{-1}y)\dd y\notag\\
 &=\int e^{U_\lambda(y)-U_\lambda(S_u^{-1}y)}
          \dd\pi_\lambda(y)\notag\\
 &\le\int e^{u\norm{\drift(y)}^2}\dd\pi_\lambda(y).
\end{align}
The case $u=0$ is immediate.
By \eqref{eq:generic-exp}, the last integral is at most $2$
whenever $uR^2\le1/256$.

Under $st(L_\lambda+R^2)\le c$, one has $tL_\lambda<1$, so the contraction map $T_t=I-t\drift$ is bi-Lipschitz by \cref{prop:pushforward}.
Write $J_t(x)=\det(I-t\hess(x))>0$ almost everywhere.
Using \eqref{eq:push-density} and changing variables $y=T_tx$, we obtain
\begin{align}\label{eq:push-norm}
 \Lp{(T_t)_\#F/\pi_\lambda}{s}^s
 &=\int
   \norm{\frac{F(x)}{J_t(x)}}^s
   \pi_\lambda(T_tx)^{1-s}J_t(x)\dx\notag\\
 &=\int
   \norm{\frac{F(x)}{\pi_\lambda(x)}}^s
   e^{(s-1)[U_\lambda(T_tx)-U_\lambda(x)]}
   J_t(x)^{1-s}\dd\pi_\lambda(x).
\end{align}
Since $tL_\lambda<1$, the smooth descent inequality gives
\[
 U_\lambda(T_tx)
 \le U_\lambda(x)
      -t\left(1-\frac{tL_\lambda}{2}\right)
       \norm{\drift(x)}^2
 \le U_\lambda(x).
\]
Thus the exponential factor in \eqref{eq:push-norm} is at most one, and Cauchy--Schwarz yields
\[
 \Lp{(T_t)_\#F/\pi_\lambda}{s}^s
 \le \Lp{F/\pi_\lambda}{2s}^s
      \left(\int J_t^{-2(s-1)}\dd\pi_\lambda\right)^{1/2}.
\]
For $k\ge1$ and $0\le z\le1/(4k)$, Bernoulli's inequality gives $(1-z)^k\ge1-kz$, and hence
\[
 (1-z)^{-k}
 \le\frac1{1-kz}
 =1+\frac{kz}{1-kz}
 \le1+2kz.
\]
Set $k=2(s-1)$.
For almost every $x$, let $\mu_1,\ldots,\mu_d$ be the eigenvalues of $\hess(x)$.
If $ktL_\lambda\le1/4$, applying the preceding inequality to $z=t\mu_i$ gives
\[
 J_t(x)^{-k}
 =\prod_{i=1}^d(1-t\mu_i)^{-k}
 \le\prod_{i=1}^d(1+2kt\mu_i)
 =\det(I+2kt\hess(x)).
\]
Since $k\le2s$, the condition $st(L_\lambda+R^2)\le c$ ensures both $ktL_\lambda\le1/4$ and $2ktR^2\le1/256$ when $c$ is sufficiently small.
Applying \eqref{eq:expanding-map} with $u=2kt$ therefore yields
\[
 \int J_t^{-k}\dd\pi_\lambda
 \le\int\det(I+2kt\hess)\dd\pi_\lambda
 \le2.
\]
Combining this with the preceding Cauchy--Schwarz bound gives
\[
 \Lp{(T_t)_\#F/\pi_\lambda}{s}
 \le2^{1/(2s)}\Lp{F/\pi_\lambda}{2s},
\]
which proves \eqref{eq:op-push}.
The changes of variables above are justified by the area formula for bi-Lipschitz maps; see \cref{prop:pushforward}.
This completes the proof of \cref{lem:operators}.

\section{Finite-Order Stationary Density Ratios}\label[appendix]{app:renyi}
This appendix proves the stationary density-ratio estimate in \cref{lem:ratios}.
The goal is to control
\[
 \sup_{0\le t\le h}
 \Lp{\nu_t/\pi_\lambda}{\alpha}
\]
under the step restrictions \eqref{eq:ratio-step}.

The proof has three steps.
In \cref{app:renyi-def}, we establish the weighted drift and curvature moments needed for the calculation.
In \cref{app:renyi-ibp}, these moments give the drift-increment bound \eqref{eq:split-drift}.
In \cref{app:renyi-close}, we insert that bound into the R\'enyi dissipation identity and use the stationary endpoint condition to obtain \eqref{eq:Renyi-final}.

The integrability needed for differentiation and integration by parts is stated in \cref{lem:prior-finiteness} and proved independently in \cref{app:finiteness}.
The regularization and limiting arguments are supplied in \cref{app:regularity}.
The interpolation method follows the standard R\'enyi analysis of Langevin algorithms \citep{ChewiEtAl}.

\subsection{Weighted energy for the density ratio}\label{app:renyi-def}
Fix $\alpha\ge2$ and $\lambda,h$ satisfying the restrictions of \cref{lem:ratios}.
We first work with an auxiliary smooth regularization.
Let $\varrho\in C_c^\infty(\R^d)$ be nonnegative and symmetric, with $\int\varrho\dx=1$, and set
\[
 \varrho_\delta(x):=\delta^{-d}\varrho(x/\delta),\qquad
 f^{(\delta)}:=f*\varrho_\delta,\qquad
 g_\lambda^{(\delta)}:=g_\lambda*\varrho_\delta,
 \qquad \delta>0.
\]
Define
\[
 U_\lambda^{(\delta)}
 :=f^{(\delta)}+g_\lambda^{(\delta)},\qquad
 \pi_\lambda^{(\delta)}
 \propto e^{-U_\lambda^{(\delta)}},\qquad
 T_t^{(\delta)}(x)
 :=x-t\nabla U_\lambda^{(\delta)}(x).
\]
The parameter $\delta$ is an auxiliary convolution scale, distinct from the Moreau parameter $\lambda$. It will tend to zero with $\lambda$ and $h$ fixed.

Let $Q_{\lambda,h}^{(\delta)}$ be the Euler kernel associated with $U_\lambda^{(\delta)}$, and let
$\widehat\pi_{\lambda,h}^{(\delta)}$ be its invariant density.
The corresponding stationary interpolation is
\[
 \nu_t^{(\delta)}
 :=\heat_t\bigl((T_t^{(\delta)})_\#
                 \widehat\pi_{\lambda,h}^{(\delta)}\bigr),
 \qquad
 \nu_0^{(\delta)}=\nu_h^{(\delta)}
 =\widehat\pi_{\lambda,h}^{(\delta)}.
\]
The structural bounds needed below hold with the same parameters $m,L_f,\tau_f,G$, and $\lambda$; their preservation is verified in \cref{app:approximation}.

For the calculation below, fix $\delta>0$ and omit the superscript $(\delta)$ from the potentials and all associated densities, kernels, and maps. In particular, $\inv$ below denotes the invariant density of the regularized Euler kernel.
We first establish the integrability needed for the calculation at this fixed $\delta$.
We then obtain a quantitative estimate independent of $\delta$, and pass to the original potential in \cref{app:approximation-limit}.

Let $X\sim\inv$ and $Z\sim N(0,I_d)$ be independent, and, for $0\le t\le h$, let
\[
 Y=X-t\drift(X)+\sqrt{2t}\,Z.
\]
The law of $Y$ is $\nu_t$.
In this appendix only, suppressing the time argument, define
\begin{gather*}
 u:=\frac{\nu_t}{\pi_\lambda},\qquad s:=\nabla\log u,\\
 e_t(y):=\drift(y)-\E[\drift(X)\mid Y=y].
\end{gather*}
The density ratio $u$ is the quantity whose $\alpha$-moment we seek to control.
Its relative score $s=\nabla\log u$ appears in the
dissipation of that moment.
The field $e_t$ measures the discrepancy between the
current drift $\drift(y)$ and the drift frozen at the
starting point $X$.
The mixed term involving $e_t$ and $s$ will be estimated
in \cref{app:renyi-close} using the drift-increment bound
proved in \cref{app:renyi-ibp}.

Before carrying out these calculations, we record the integrability statement that makes them legitimate.

\begin{lemma}[Integrability before dissipation]
\label{lem:prior-finiteness}
Consider the regularized interpolation defined above, with $\delta>0$ fixed.
Under the restrictions of \cref{lem:ratios}, $\nu_t$ is positive, locally $C^{1,2}$ in $(t,y)$ for $0<t<h$, and pointwise continuous at both time endpoints.
Moreover,
\[
 \sup_{0\le t\le h}\int
 \bigl[u^\alpha(1+\norm s^2+\norm{e_t}^2)
       +u^{3\alpha/2}\bigr]\dd\pi_\lambda<\infty.
\]
For every integer $k\ge0$, one also has
\[
 \sup_{0\le t\le h}\E\!\left[
 u(Y)^{\alpha-1}(1+\norm X+\norm Y)^k
 \right]<\infty.
\]
These bounds are uniform in $t\in[0,h]$.
Their finite values may depend on the fixed parameters and on $\delta$; no uniform bound as $\delta\downarrow0$ is asserted in this lemma.
\end{lemma}

The proof of \cref{lem:prior-finiteness} is given in
\cref{app:finiteness}.
It establishes the required integrability independently of the dissipation calculation.
The finite bounds from that proof justify the operations below; their numerical values do not enter the quantitative estimate.

We now prepare two moment bounds for the calculation in \cref{app:renyi-ibp}:
an $L^2$ bound on $\nabla f$ and a bound on the mean of $\Delta g_\lambda$ under the appropriate density-ratio weight.
We define
\begin{equation}\label{eq:escort}
 \omega(\mathrm dx)
 :=\frac{u(x)^\alpha\pi_\lambda(x)\dx}
          {\int u^\alpha\dd\pi_\lambda},
 \qquad
 I:=\E_\omega\norm s^2.
\end{equation}

\emph{A weighted moment of $\nabla f$.}
Recall that $\drift=\nabla f+\nabla g_\lambda$ and
$\nabla\log\omega=\alpha s-\drift$.
Integration by parts gives
\begin{align*}
 \E_\omega\Delta f
 &=-\E_\omega\ip{\nabla f}{\nabla\log\omega}\\
 &=\E_\omega\norm{\nabla f}^2
   +\E_\omega\ip{\nabla f}{\nabla g_\lambda}
   -\alpha\E_\omega\ip{\nabla f}{s}.
\end{align*}
The integration by parts is justified by spatial cutoffs
and the moment and score bounds in
\cref{lem:prior-finiteness}.
Rearranging the preceding identity, using
$\Delta f\le\tau_f$ and $\norm{\nabla g_\lambda}\le G$,
and applying Cauchy--Schwarz, we obtain
\begin{align*}
 \norm{\nabla f}_{L^2(\omega)}^2
 &=\E_\omega\Delta f
   -\E_\omega\ip{\nabla f}{\nabla g_\lambda}
   +\alpha\E_\omega\ip{\nabla f}{s}\\
 &\le\tau_f
   +G\E_\omega\norm{\nabla f}
   +\alpha\norm{\nabla f}_{L^2(\omega)}\sqrt I\\
 &\le\tau_f
   +(G+\alpha\sqrt I)\norm{\nabla f}_{L^2(\omega)}.
\end{align*}

Using $R^2=\tau_f+GR$, we solve this inequality and obtain
\begin{equation}\label{eq:escort-drift}
 \norm{\nabla f}_{L^2(\omega)}
 \le R+\alpha\sqrt I.
\end{equation}

\emph{The weighted mean of the Moreau curvature.}
We next estimate $\E_\omega\Delta g_\lambda$. Applying integration by parts directly to
$\nabla g_\lambda$ gives
\begin{align}
 \E_\omega\Delta g_\lambda
 &=-\E_\omega
      \ip{\nabla g_\lambda}{\nabla\log\omega}\notag\\
 &=\E_\omega
      \ip{\nabla g_\lambda}
         {\nabla f+\nabla g_\lambda-\alpha s}\notag\\
 &\le G\norm{\nabla f}_{L^2(\omega)}
       +G^2+\alpha G\sqrt I\notag\\
 &\le GR+G^2+2\alpha G\sqrt I\notag\\
 &\le2GR+2\alpha G\sqrt I.
 \label{eq:escort-curvature}
\end{align}
The first inequality uses $\norm{\nabla g_\lambda}\le G$ and Cauchy--Schwarz.
The next line uses \eqref{eq:escort-drift}, and the last uses $G^2\le GR$.

\subsection{Gaussian integration by parts with the density-ratio weight}\label{app:renyi-ibp}

This subsection proves \eqref{eq:split-drift}.

For an integrable function $\Phi(X,Y)$, define the weighted joint expectation
\[
 \E_\alpha[\Phi(X,Y)]
 :=
 \frac{\E\!\left[u(Y)^{\alpha-1}\Phi(X,Y)\right]}
      {\int u^\alpha\dd\pi_\lambda}.
\]
Since $Y$ has density $\nu_t=u\pi_\lambda$, the denominator normalizes this expectation, and its $Y$ marginal is $\omega$:
\[
 \E_\alpha[\phi(Y)]
 =
 \frac{\int\phi(y)u(y)^\alpha\pi_\lambda(y)\dd y}
      {\int u^\alpha\dd\pi_\lambda}
 =\E_\omega\phi
\]
for every integrable test function $\phi$.

Fix $0<t\le h$.
Let $w$ be the gradient of a convex function with $L$-Lipschitz gradient, and write
\[
 \Delta w:=w(Y)-w(X).
\]
Cocoercivity and
$Y-X=-t\drift(X)+\sqrt{2t}\,Z$ give
\[
 \E_\alpha\norm{\Delta w}^2
 \le
 -tL\,\E_\alpha\ip{\Delta w}{\drift(X)}
 +L\,\E_\alpha\ip{\Delta w}{\sqrt{2t}\,Z}.
\]

To evaluate the Gaussian term, condition on $X$ and integrate by parts under the original standard Gaussian law of $Z$, keeping $u(Y)^{\alpha-1}$ in the integrand. The identities
\[
 \partial_{Z_i}Y_j=\sqrt{2t}\,\delta_{ij},
 \qquad
 \nabla_y u(y)^{\alpha-1}
 =(\alpha-1)u(y)^{\alpha-1}s(y)
\]
yield
\[
 \E_\alpha\ip{\Delta w}{\sqrt{2t}\,Z}
 =
 2t\,\E_\omega\Div w
 +2t(\alpha-1)\E_\alpha\ip{\Delta w}{s(Y)}.
\]
The joint weighted moment bound with $k=2$ in \cref{lem:prior-finiteness} and $\norm{\Delta w}\le L\norm{Y-X}$ give $\E_\alpha\norm{\Delta w}^2<\infty$. The same lemma gives $I<\infty$. Since the $Y$ marginal is $\omega$, Cauchy--Schwarz yields
\[
 \E_\alpha\left|\ip{\Delta w}{s(Y)}\right|
 \le
 \left(\E_\alpha\norm{\Delta w}^2\right)^{1/2}
 \sqrt I
 <\infty.
\]
The remaining terms are integrable by the joint moment bound and the bounded derivative of $w$. These bounds justify removing the cutoffs in the Gaussian integration by parts.

Substitution gives
\begin{align}\label{eq:weighted-cocoercivity}
 \E_\alpha\norm{w(Y)-w(X)}^2
 \le tL\Bigl[&
 -\E_\alpha\ip{w(Y)-w(X)}{\drift(X)}
 +2\E_\omega\Div w\notag\\
 &+2(\alpha-1)
   \E_\alpha\ip{w(Y)-w(X)}{s(Y)}
 \Bigr].
\end{align}

\emph{The smooth component.}
We first estimate the squared increment of $\nabla f$.
Define
\[
 D_f:=\E_\alpha
       \norm{\nabla f(Y)-\nabla f(X)}^2.
\]
To use \eqref{eq:weighted-cocoercivity}, we also need a weighted second moment of $\drift(X)$.
The identity
\[
 \drift(X)
 =\nabla f(Y)
  -\bigl(\nabla f(Y)-\nabla f(X)\bigr)
  +\nabla g_\lambda(X)
\]
and the fact that the $Y$ marginal of $\E_\alpha$ is $\omega$ give
\begin{align*}
 \bigl(\E_\alpha\norm{\drift(X)}^2\bigr)^{1/2}
 \le
 \norm{\nabla f}_{L^2(\omega)}+\sqrt{D_f}+G
 \le2R+\alpha\sqrt I+\sqrt{D_f},
\end{align*}
where the last line uses \eqref{eq:escort-drift}
and $G\le R$.

Apply \eqref{eq:weighted-cocoercivity} with
$w=\nabla f$ and $L=L_f$.
Using $\Div(\nabla f)=\Delta f\le\tau_f$ and
Cauchy--Schwarz, we obtain
\begin{align*}
 D_f
 &\le tL_f\Bigl[
 \sqrt{D_f}(2R+\alpha\sqrt I+\sqrt{D_f})
 +2\tau_f
 +2(\alpha-1)\sqrt{D_f}\sqrt I
 \Bigr]\\
 &\le tL_f\Bigl[
 D_f+(2R+3\alpha\sqrt I)\sqrt{D_f}+2\tau_f
 \Bigr].
\end{align*}
Young's inequality gives
\[
 tL_f(2R+3\alpha\sqrt I)\sqrt{D_f}
 \le\frac14D_f+8t^2L_f^2R^2
                      +18t^2L_f^2\alpha^2I.
\]
For $tL_f\le1/4$, moving the terms involving $D_f$ to the left therefore yields
\[
 \frac12D_f
 \le2tL_f\tau_f+8t^2L_f^2R^2
                    +18t^2L_f^2\alpha^2I.
\]
In particular, we may use
\begin{equation}\label{eq:Df}
 D_f\le6tL_f\tau_f+16t^2L_f^2R^2
                       +36t^2L_f^2\alpha^2I.
\end{equation}
Since $\tau_f\le R^2$ and $tL_f\le1/4$, the preceding estimates also give
\[
 \bigl(\E_\alpha\norm{\drift(X)}^2\bigr)^{1/2}
 \le6R+3\alpha\sqrt I.
\]

\emph{The Moreau component.}
We next estimate
\[
 D_g:=\E_\alpha
       \norm{\nabla g_\lambda(Y)-\nabla g_\lambda(X)}^2.
\]
Apply \eqref{eq:weighted-cocoercivity} with $w=\nabla g_\lambda$ and $L=\lambda^{-1}$.
The bounded-gradient property gives
\[
 \norm{\nabla g_\lambda(Y)-\nabla g_\lambda(X)}
 \le2G.
\]
It follows that
\begin{align*}
 D_g
 &\le\frac t\lambda\Bigl[
 2G\bigl(\E_\alpha\norm{\drift(X)}^2\bigr)^{1/2}
 +2\E_\omega\Delta g_\lambda
 +4(\alpha-1)G\sqrt I
 \Bigr]\\
 &\le\frac t\lambda\Bigl[
 2G(6R+3\alpha\sqrt I)
 +2(2GR+2\alpha G\sqrt I)
 +4(\alpha-1)G\sqrt I
 \Bigr].
\end{align*}
The second line uses the bound above and \eqref{eq:escort-curvature}.
Collecting terms gives
\begin{equation}\label{eq:Da}
 D_g\le\frac t\lambda
       \bigl(16GR+14\alpha G\sqrt I\bigr).
\end{equation}

\emph{Combining the two components.}
Since $\drift=\nabla f+\nabla g_\lambda$,
\[
 \E_\alpha\norm{\drift(Y)-\drift(X)}^2
 \le2D_f+2D_g.
\]
Equations \eqref{eq:Df} and \eqref{eq:Da} therefore yield
\begin{align}\label{eq:split-drift}
 \E_\alpha\norm{\drift(Y)-\drift(X)}^2
 \le{}&12tL_f\tau_f+32t^2L_f^2R^2
       +72t^2L_f^2\alpha^2I\notag\\
 &+32(t/\lambda)GR
       +28(t/\lambda)\alpha G\sqrt I.
\end{align}

\subsection{Proof of the density-ratio bound}
\label{app:renyi-close}

This subsection proves \eqref{eq:Renyi-final}, which implies \eqref{eq:ratios}.

\paragraph{Step 1. Evolution equation for $\nu_t$: \eqref{eq:Euler-PDE}.}

Recall that
\[
 Y=X-t\drift(X)+\sqrt{2t}\,Z.
\]
For $\phi\in C_c^\infty(\R^d)$ and $0<t<h$, differentiation under the expectation gives
\[
 \frac\dd{\dd t}\int\phi(y)\nu_t(y)\dd y
 =\E\ip{\nabla\phi(Y)}
              {-\drift(X)+Z/\sqrt{2t}}.
\]
This differentiation is valid on compact subintervals of $(0,h)$ because $\nabla\phi$ is bounded, the drift grows at most linearly, and $X$ has finite moments by \cref{lem:prior-finiteness}.

For each coordinate $i$, Gaussian integration by parts, first conditional on $X$, gives
\[
 \E[Z_i\partial_i\phi(Y)]
 =\sqrt{2t}\,\E[\partial_{ii}\phi(Y)].
\]
Consequently,
\begin{align*}
 \frac\dd{\dd t}\int\phi(y)\nu_t(y)\dd y
 &=\E\bigl[
       \Delta\phi(Y)
       -\ip{\drift(X)}{\nabla\phi(Y)}
      \bigr]\\
 &=\int
   \bigl[
     \Delta\phi(y)
     -\ip{\drift(y)-e_t(y)}{\nabla\phi(y)}
   \bigr]\nu_t(y)\dd y,
\end{align*}
where the second equality uses
\[
 \E[\drift(X)\mid Y=y]=\drift(y)-e_t(y).
\]
Thus, in the weak sense,
\[
 \partial_t\nu_t
 =\Delta\nu_t+\Div(\nu_t\drift)-\Div(\nu_te_t).
\]
Since $\nabla\log\pi_\lambda=-\drift$, we have
\[
 \nabla\nu_t+\nu_t\drift
 =\nu_t\nabla\log\frac{\nu_t}{\pi_\lambda}.
\]
This proves
\begin{equation}\label{eq:Euler-PDE}
 \partial_t\nu_t
 =\Div\left(
        \nu_t\nabla\log\frac{\nu_t}{\pi_\lambda}
       \right)
  -\Div(\nu_te_t).
\end{equation}

\paragraph{Step 2. Derivative of the R\'enyi divergence: \eqref{eq:Renyi-derivative}.}

By \eqref{eq:renyi-def}, with $u=\nu_t/\pi_\lambda$,
\[
 \mathcal R_\alpha(\nu_t\Vert\pi_\lambda)
 =\frac1{\alpha-1}
       \log\int u^\alpha\dd\pi_\lambda.
\]
We first compute the derivative of the integral on the right-hand side.

Choose $\chi_n\in C_c^\infty(\R^d)$ such that $0\le\chi_n\le1$, $\chi_n=1$ on the ball of radius $n$, $\chi_n=0$ outside the ball of radius $2n$, and $\norm{\nabla\chi_n}_\infty\le C/n$.
Fix $0<t_0<t_1<h$.
For each fixed $t\in[t_0,t_1]$, Step 1 gives
\[
 \int\phi\,\partial_t\nu_t\dx
 =-\int\ip{\nabla\phi}{s-e_t}\nu_t\dx,
 \qquad \phi\in C_c^\infty(\R^d).
\]
By \cref{lem:prior-finiteness}, $u$ is positive and locally $C^{1,2}$, $\partial_t\nu_t$ is locally continuous, and $\nu_t(s-e_t)$ is locally integrable.
Approximation in $C^1$, with support in a fixed compact set, therefore extends this identity to $\phi\in C_c^1(\R^d)$.

At this fixed time, we may take $\phi=\chi_nu^{\alpha-1}$.
Applying the chain rule first, and then the spatial identity above, gives
\begin{align*}
 \frac\dd{\dd t}\int\chi_nu^\alpha\dd\pi_\lambda
 &=\alpha\int
       \chi_nu^{\alpha-1}\partial_t\nu_t\dx\\
 &=-\alpha\int
       \ip{\nabla(\chi_nu^{\alpha-1})}{s-e_t}
       \nu_t\dx.
\end{align*}

Since
\[
 \nabla(\chi_nu^{\alpha-1})
 =u^{\alpha-1}\nabla\chi_n
  +(\alpha-1)\chi_nu^{\alpha-1}s,
\]
we obtain
\begin{align*}
 \frac\dd{\dd t}\int\chi_nu^\alpha\dd\pi_\lambda
 ={}&-\alpha(\alpha-1)
       \int\chi_nu^\alpha\norm s^2\dd\pi_\lambda\\
 &+\alpha(\alpha-1)
       \int\chi_nu^\alpha\ip{e_t}{s}\dd\pi_\lambda\\
 &-\alpha\int
       u^\alpha\ip{s-e_t}{\nabla\chi_n}\dd\pi_\lambda.
\end{align*}
Keep $t_0$ and $t_1$ fixed and integrate the preceding identity in time. Writing $u(t_j)=\nu_{t_j}/\pi_\lambda$ at the endpoints gives
\begin{align*}
 &\int\chi_nu(t_1)^\alpha\dd\pi_\lambda
  -\int\chi_nu(t_0)^\alpha\dd\pi_\lambda\\
 &\quad=-\alpha(\alpha-1)
       \int_{t_0}^{t_1}\int
          \chi_nu^\alpha\norm s^2
          \dd\pi_\lambda\dd t\\
 &\quad+\alpha(\alpha-1)
       \int_{t_0}^{t_1}\int
          \chi_nu^\alpha\ip{e_t}{s}
          \dd\pi_\lambda\dd t\\
 &\quad-\alpha
       \int_{t_0}^{t_1}\int
          u^\alpha\ip{s-e_t}{\nabla\chi_n}
          \dd\pi_\lambda\dd t.
\end{align*}

We now remove the spatial cutoff while keeping $t_0$ and $t_1$ fixed.
The last term satisfies
\begin{align*}
 \alpha\left|
   \int_{t_0}^{t_1}\int
      u^\alpha\ip{s-e_t}{\nabla\chi_n}
      \dd\pi_\lambda\dd t
  \right|
 \le
 \frac{C\alpha}{n}
 \int_{t_0}^{t_1}\int
      u^\alpha
      \bigl(1+\norm s^2+\norm{e_t}^2\bigr)
      \dd\pi_\lambda\dd t.
\end{align*}
The time integral on the right is finite and independent of $n$ by \cref{lem:prior-finiteness}.
Moreover,
\[
 u^\alpha|\ip{e_t}{s}|
 \le\frac12u^\alpha
          \bigl(\norm{e_t}^2+\norm s^2\bigr),
\]
so the same lemma provides an integrable bound for each of the other time integrands.
We may therefore let $n\to\infty$ by dominated convergence, obtaining
\begin{align*}
 &\int u(t_1)^\alpha\dd\pi_\lambda
  -\int u(t_0)^\alpha\dd\pi_\lambda\\
 &\quad=-\alpha(\alpha-1)
       \int_{t_0}^{t_1}\int
          u^\alpha\norm s^2
          \dd\pi_\lambda\dd t\\
 &\qquad\quad+\alpha(\alpha-1)
       \int_{t_0}^{t_1}\int
          u^\alpha\ip{e_t}{s}
          \dd\pi_\lambda\dd t.
\end{align*}

It remains to include the time endpoints.
By \cref{lem:prior-finiteness}, both time integrands in the last identity belong to $L^1(0,h)$.
The endpoint limits \eqref{eq:finite-moment-continuity}, proved in \cref{app:finiteness}, show that $t\mapsto\int u(t)^\alpha\dd\pi_\lambda$ is continuous at $t=0$ and $t=h$.
We can therefore let $t_0\downarrow0$ or $t_1\uparrow h$ in this identity, and then include both endpoints. The same identity consequently holds for every $0\le t_0<t_1\le h$.

Since its right-hand side is the time integral of an $L^1(0,h)$ function, this identity proves that $t\mapsto\int u(t)^\alpha\dd\pi_\lambda$ is absolutely continuous on $[0,h]$.
Its derivative is therefore
\begin{align*}
 \frac\dd{\dd t}\int u^\alpha\dd\pi_\lambda
 ={}-\alpha(\alpha-1)
       \int u^\alpha\norm s^2\dd\pi_\lambda
 +\alpha(\alpha-1)
       \int u^\alpha\ip{e_t}{s}\dd\pi_\lambda
\end{align*}
for almost every $t\in(0,h)$.

Because $\int u^\alpha\dd\pi_\lambda\ge1$, the R\'enyi divergence is also absolutely continuous on $[0,h]$. The chain rule gives
\begin{align*}
 \frac\dd{\dd t}
       \mathcal R_\alpha(\nu_t\Vert\pi_\lambda)
 =\frac{
       \frac\dd{\dd t}\int u^\alpha\dd\pi_\lambda
      }{
       (\alpha-1)\int u^\alpha\dd\pi_\lambda
      }
 =-\alpha I+\alpha\E_\omega\ip{e_t}{s}.
\end{align*}

The weight defining $\E_\alpha$ depends only on $Y$, so it leaves the conditional law of $X$ given $Y$ unchanged. Hence
\[
 e_t(Y)
 =\E_\alpha[
     \drift(Y)-\drift(X)\mid Y
   ].
\]
Conditional Jensen's inequality and the fact that the $Y$ marginal is $\omega$ give
\[
 \E_\omega\norm{e_t}^2
 \le\E_\alpha
      \norm{\drift(Y)-\drift(X)}^2.
\]
Using
\[
 \ip{e_t}{s}
 \le\frac12\norm s^2+\frac12\norm{e_t}^2,
\]
we conclude that, for almost every $t$,
\begin{align}\label{eq:Renyi-derivative}
 \frac\dd{\dd t}
       \mathcal R_\alpha(\nu_t\Vert\pi_\lambda)
 &=-\alpha I+\alpha\E_\omega\ip{e_t}{s}
       \notag\\
 &\le-\frac\alpha2 I
       +\frac\alpha2\E_\omega\norm{e_t}^2
       \notag\\
 &\le-\frac\alpha2 I
       +\frac\alpha2\E_\alpha
          \norm{\drift(Y)-\drift(X)}^2.
\end{align}

\paragraph{Step 3. Applying the drift bound: \eqref{eq:Renyi-ode}.}

Substituting \eqref{eq:split-drift} into
\eqref{eq:Renyi-derivative} gives
\begin{align*}
 \frac\dd{\dd t}
       \mathcal R_\alpha(\nu_t\Vert\pi_\lambda)
 \le{}&-\frac\alpha2 I
       +6\alpha tL_f\tau_f
       +16\alpha t^2L_f^2R^2\\
 &+36\alpha^3t^2L_f^2I
       +16\alpha(t/\lambda)GR
       +14\alpha^2(t/\lambda)G\sqrt I.
\end{align*}
Since $\tau_f\ge m$, the first condition in \eqref{eq:ratio-step} implies $hL_f\le c/\alpha^2$.
Taking the universal constant $c$ small enough gives, for $0\le t\le h$,
\[
 36\alpha^3t^2L_f^2I\le\frac\alpha8 I.
\]
Also, Young's inequality gives
\[
 14\alpha^2(t/\lambda)G\sqrt I
 \le\frac\alpha8 I
       +392\alpha^3(t/\lambda)^2G^2.
\]
Combining these estimates proves
\begin{align}\label{eq:Renyi-ode}
 \frac\dd{\dd t}
       \mathcal R_\alpha(\nu_t\Vert\pi_\lambda)
 \le{}-\frac\alpha4 I
       +C\alpha\Bigl[
          tL_f\tau_f+t^2L_f^2R^2
          +(t/\lambda)GR
          +\alpha^2(t/\lambda)^2G^2
       \Bigr].
\end{align}

\paragraph{Step 4. A lower bound for $I$: \eqref{eq:Renyi-energy}.}

By \cref{lem:prior-finiteness},
\[
 \int u^\alpha\dd\pi_\lambda<\infty,
 \qquad
 \int\norm{\nabla u^{\alpha/2}}^2\dd\pi_\lambda
 =\frac{\alpha^2}{4}
       \int u^\alpha\norm s^2\dd\pi_\lambda<\infty.
\]
We may therefore apply \eqref{eq:LSI} to $w=u^{\alpha/2}$. Dividing by $\int u^\alpha\dd\pi_\lambda$ gives
\[
 \frac{\Ent_{\pi_\lambda}(u^\alpha)}
      {\int u^\alpha\dd\pi_\lambda}
 \le\frac2m
       \frac{
         \int\norm{\nabla u^{\alpha/2}}^2\dd\pi_\lambda
       }{
         \int u^\alpha\dd\pi_\lambda
       }
 =\frac{\alpha^2}{2m}I.
\]

We next bound the left-hand side from below.
Since $\int u\dd\pi_\lambda=1$, the definition of
$\omega$ gives
\[
 \E_\omega[u^{1-\alpha}]
 =\frac{\int u\dd\pi_\lambda}
        {\int u^\alpha\dd\pi_\lambda}
 =\frac1{\int u^\alpha\dd\pi_\lambda}.
\]
The $3\alpha/2$ moment in
\cref{lem:prior-finiteness} ensures that
$\E_\omega|\log u|<\infty$:
$u^\alpha|\log u|$ is bounded when $0<u\le1$,
and is at most $C_\alpha u^{3\alpha/2}$ when $u\ge1$.
Jensen's inequality for the logarithm therefore gives
\begin{align*}
 (1-\alpha)\E_\omega\log u
 &=\E_\omega[\log(u^{1-\alpha})]\\
 &\le\log\E_\omega[u^{1-\alpha}]
 =-\log\int u^\alpha\dd\pi_\lambda.
\end{align*}
Because $\alpha>1$, this implies
\[
 \E_\omega\log u
 \ge\frac1{\alpha-1}
        \log\int u^\alpha\dd\pi_\lambda
 =\mathcal R_\alpha(\nu_t\Vert\pi_\lambda).
\]
Consequently,
\begin{align*}
 \frac{\Ent_{\pi_\lambda}(u^\alpha)}
      {\int u^\alpha\dd\pi_\lambda}
 &=\alpha\E_\omega\log u
       -\log\int u^\alpha\dd\pi_\lambda\\
 &\ge\alpha\mathcal R_\alpha(\nu_t\Vert\pi_\lambda)
       -(\alpha-1)
          \mathcal R_\alpha(\nu_t\Vert\pi_\lambda)\\
 &=\mathcal R_\alpha(\nu_t\Vert\pi_\lambda).
\end{align*}
Combining the upper and lower bounds proves
\begin{equation}\label{eq:Renyi-energy}
 I\ge\frac{2m}{\alpha^2}
          \mathcal R_\alpha(\nu_t\Vert\pi_\lambda).
\end{equation}

\paragraph{Step 5. Uniform $L^\alpha$ bound for the density ratio: \eqref{eq:Renyi-final} and \eqref{eq:ratios}.}

In \eqref{eq:Renyi-ode}, use \eqref{eq:Renyi-energy} and bound $t$ by $h$ in the four nonnegative terms.
Multiplying the resulting inequality by $e^{mt/(2\alpha)}$ and integrating from $0$ to $t$ gives, after enlarging the universal constant $C$,
\begin{align*}
 \mathcal R_\alpha(\nu_t\Vert\pi_\lambda)
 \le{}&e^{-mt/(2\alpha)}
          \mathcal R_\alpha(\nu_0\Vert\pi_\lambda)\\
 &+\frac{C\alpha^2}{m}
       \bigl(1-e^{-mt/(2\alpha)}\bigr)
       \Bigl[
          hL_f\tau_f+h^2L_f^2R^2
          +(h/\lambda)GR
          +\alpha^2(h/\lambda)^2G^2
       \Bigr].
\end{align*}
Taking $t=h$ and using $\nu_h=\nu_0$, we move $e^{-mh/(2\alpha)}\mathcal R_\alpha(\nu_0\Vert\pi_\lambda)$ to the left-hand side and cancel the positive factor $1-e^{-mh/(2\alpha)}$.
This bounds $\mathcal R_\alpha(\nu_0\Vert\pi_\lambda)$ by $C\alpha^2/m$ times the bracket above.
Substituting this bound back into the inequality for general $t$ proves
\begin{align}\label{eq:Renyi-final}
 \sup_{0\le t\le h}
       \mathcal R_\alpha(\nu_t\Vert\pi_\lambda)
 \le{}\frac{C\alpha^2}{m}
       \Bigl[
          hL_f\tau_f+h^2L_f^2R^2
          +(h/\lambda)GR
          +\alpha^2(h/\lambda)^2G^2
       \Bigr].
\end{align}

The conditions in \eqref{eq:ratio-step}, together with $R^2\ge m$, give
\begin{align*}
 \frac{\alpha^2}{m}hL_f\tau_f
 &\le c,\\
 \frac{\alpha^2}{m}h^2L_f^2R^2
 &=\left(
     \frac{\alpha hL_fR}{\sqrt m}
   \right)^2
 \le c^2,\\
 \frac{\alpha^2}{m}(h/\lambda)GR
 &\le c,\\
 \frac{\alpha^4}{m}(h/\lambda)^2G^2
 &=\frac m{R^2}
   \left(
     \frac{\alpha^2}{m}(h/\lambda)GR
   \right)^2
 \le c^2.
\end{align*}
Taking the universal constant $c$ small enough makes the right-hand side of \eqref{eq:Renyi-final} at most $1/4$.
Finally, \eqref{eq:renyi-def} gives
\[
 \sup_{0\le t\le h}
       \Lp{\nu_t/\pi_\lambda}{\alpha}
 =
 \exp\left\{
   \frac{\alpha-1}{\alpha}
   \sup_{0\le t\le h}
       \mathcal R_\alpha(\nu_t\Vert\pi_\lambda)
 \right\}
 \le e^{1/4}.
\]
This is \eqref{eq:ratios}.

All quantitative constants above are independent of the auxiliary smoothing scale. Passing to the limit as described in \cref{app:regularity} gives \eqref{eq:ratios} for the original potential. This completes the proof of \cref{lem:ratios}.

\section{Poisson Regularity and Negative Sobolev Norms}
\label[appendix]{app:poisson}
This appendix proves \cref{lem:poisson-energy} and the
test-function equivalence \eqref{eq:Hminus-test-classes}.
Throughout, $\lambda$ is fixed, $U_\lambda\in C^{1,1}(\R^d)$, and
$mI\preceq\hess\preceq L_\lambda I$ almost everywhere.

\subsection{Existence and second-order energy}\label{app:poisson-energy}
\begin{proof}[Proof of the existence and energy assertions in \cref{lem:poisson-energy}]
By \eqref{eq:Poincare}, the gradient norm is a complete Hilbert norm on the mean-zero subspace of $H^1(\pi_\lambda)$. For $\phi\in L^2(\pi_\lambda)$, the functional
\[
 v\longmapsto-\int\left(\phi-\int\phi\dd\pi_\lambda\right)v
                          \dd\pi_\lambda
\]
is continuous in that norm, with norm at most $m^{-1/2}\Lp{\phi-\int\phi\dd\pi_\lambda}{2}$. The Riesz representation theorem gives a unique mean-zero $\psi\in H^1(\pi_\lambda)$ satisfying \eqref{eq:poisson-weak}.
The identity extends from mean-zero tests to every $v\in H^1(\pi_\lambda)$ because both sides are unchanged when a constant is added to $v$.

On compact sets, $\pi_\lambda$ is bounded above and below by positive constants and $\drift$ is bounded.
The weak equation therefore gives, in the sense of distributions,
\[
 \Delta\psi=\phi-\int\phi\dd\pi_\lambda
                   +\ip{\drift}{\nabla\psi}
 \in L^2_{\mathrm{loc}}(\R^d).
\]
For $\eta\in C_c^\infty(\R^d)$, the function $\eta\psi$ belongs to ordinary $H^1(\R^d)$ and
\[
 \Delta(\eta\psi)
 =\eta\Delta\psi+2\ip{\nabla\eta}{\nabla\psi}+\psi\Delta\eta
 \in L^2(\R^d).
\]
Set $u=\eta\psi$ and let $u_\varepsilon=u*\varrho_\varepsilon$, where $\varrho_\varepsilon(x)=\varepsilon^{-d}\varrho(x/\varepsilon)$, $\varrho\in C_c^\infty(B(0,1))$, $\varrho\ge0$, and $\int\varrho\dx=1$.
Then $u_\varepsilon\in C_c^\infty(\R^d)$, $u_\varepsilon\to u$ in $H^1(\R^d)$, and
\[
 \Delta u_\varepsilon
 =(\Delta u)*\varrho_\varepsilon
 \longrightarrow\Delta u
 \quad\text{in }L^2(\R^d).
\]
For every $w\in C_c^\infty(\R^d)$, two integrations by parts give
\[
 \int\norm{\nabla^2w}_F^2\dx
 =\sum_{i,j}\int(\partial_{ii}w)(\partial_{jj}w)\dx
 =\int|\Delta w|^2\dx.
\]
Applying this identity to $w=u_\varepsilon-u_{\varepsilon'}$ shows that $\nabla^2u_\varepsilon$ is Cauchy in $L^2(\R^d)$. Its limit is the weak Hessian of $u$, as follows by passing to the limit in
\[
 \int\partial_{ij}u_\varepsilon\,\zeta\dx
 =\int u_\varepsilon\,\partial_{ij}\zeta\dx,
 \qquad \zeta\in C_c^\infty(\R^d).
\]
Thus $u\in H^2(\R^d)$. Since $\eta$ was arbitrary, $\psi\in H^2_{\mathrm{loc}}(\R^d)$.

Now let $\phi\in C_c^\infty(\R^d)$.
For $u\in C_c^\infty(\R^d)$, differentiation of the generator in the weak
sense gives
\[
 \nabla\gen u=\gen(\nabla u)-(\hess)\nabla u,
\]
where $\gen$ acts componentwise on $\nabla u$.
Weighted integration by parts then yields
\begin{align}\label{eq:Bochner-compact}
 \int(\gen u)^2\dd\pi_\lambda
 &=-\int\ip{\nabla u}{\nabla\gen u}\dd\pi_\lambda\notag\\
 &=\int\norm{\nabla^2u}_F^2\dd\pi_\lambda
   +\int\nabla u^\top(\hess)\nabla u\dd\pi_\lambda.
\end{align}
Only the bounded weak derivative of $\drift$ is used.
For a compactly supported $u\in H^2(\R^d)$, let $u_\varepsilon=u*\varrho_\varepsilon$. Mollification commutes with weak derivatives through order two, so $u_\varepsilon\to u$ in $H^2(\R^d)$. For $0<\varepsilon<1$, these functions are supported in a fixed compact set, on which $\pi_\lambda$ and $\drift$ are bounded. Consequently,
\[
 \Lp{\gen u_\varepsilon-\gen u}{2}
 +\Lp{\nabla^2u_\varepsilon-\nabla^2u}{2}
 +\Lp{\nabla u_\varepsilon-\nabla u}{2}
 \longrightarrow0.
\]
Together with the boundedness of $\hess$, these convergences justify passage to the limit in \eqref{eq:Bochner-compact}. The identity therefore holds for compactly supported $H^2(\R^d)$ functions.

Choose smooth cutoffs $0\le\chi_n\le1$, equal to one for $\norm x\le n$ and zero for $\norm x\ge2n$, such that $\norm{\nabla\chi_n}\le C/n$ and $\norm{\nabla^2\chi_n}_F\le C_d/n^2$. Local regularity implies that $\chi_n\psi$ is a compactly supported $H^2$ function.
Since $\drift$ has at most linear growth, $\gen\chi_n$ is bounded uniformly in $n$ on its annular support, with the model parameters fixed.
Consequently,
\begin{align*}
 \gen(\chi_n\psi)
 &=\chi_n\left(\phi-\int\phi\dd\pi_\lambda\right)
   +2\ip{\nabla\chi_n}{\nabla\psi}+\psi\gen\chi_n\\
 &\longrightarrow\phi-\int\phi\dd\pi_\lambda
   =\gen\psi
 \qquad\text{in }L^2(\pi_\lambda).
\end{align*}
The first term converges by dominated convergence; the other two tend to zero because $\psi,\nabla\psi\in L^2(\pi_\lambda)$ and the derivatives of $\chi_n$ are supported where $\norm x\ge n$.

Apply \eqref{eq:Bochner-compact} to $\chi_n\psi$. Since $\hess\succeq0$, both energy terms are nonnegative. Moreover, $\chi_n\psi=\psi$ on $B(0,n)$, so their derivatives agree there. Fatou's lemma and the convergence of $\gen(\chi_n\psi)$ give
\begin{align*}
 &\int\norm{\nabla^2\psi}_F^2\dd\pi_\lambda
   +\int\nabla\psi^\top\hess\nabla\psi\dd\pi_\lambda\\
 &\qquad\le
   \liminf_{n\to\infty}
   \int|\gen(\chi_n\psi)|^2\dd\pi_\lambda\\
 &\qquad=
   \int\left|\phi-\int\phi\dd\pi_\lambda\right|^2
       \dd\pi_\lambda
   <\infty.
\end{align*}
In particular, $\nabla^2\psi\in L^2(\pi_\lambda)$. This proves the first inequality in \eqref{eq:bochner-main}; the second follows from \eqref{eq:Poincare} applied to $\phi$.
\end{proof}

\subsection{Approximation by compactly supported tests}
\label{app:residual-pairing}
\begin{proof}[Proof of the approximation assertion in
\cref{lem:poisson-energy}]
Let $\phi\in C_c^\infty(\R^d)$ and use the cutoffs from the preceding proof. The convergence of $\gen(\chi_n\psi)$ to $\gen\psi$ in $L^2(\pi_\lambda)$ has already been established. Now that \eqref{eq:bochner-main} gives global weighted integrability of $\nabla^2\psi$, the product formula
\begin{align*}
 \nabla^2(\chi_n\psi)-\nabla^2\psi
 &=(\chi_n-1)\nabla^2\psi
   +\nabla\chi_n\otimes\nabla\psi\\
 &\quad+\nabla\psi\otimes\nabla\chi_n
   +\psi\nabla^2\chi_n
\end{align*}
implies convergence to zero in $L^2(\pi_\lambda)$.
Likewise,
\[
 \nabla(\chi_n\psi)-\nabla\psi
 =(\chi_n-1)\nabla\psi+\psi\nabla\chi_n
 \longrightarrow0
 \quad\text{in }L^2(\pi_\lambda).
\]

For each fixed $n$, mollify $\chi_n\psi$. The mollifications converge in ordinary $H^2(\R^d)$ and remain supported in a fixed compact set. On that set, $\pi_\lambda$ is bounded above and below by positive constants and $\drift$ is bounded. Thus they also converge in each of the three norms in \eqref{eq:poisson-approximation}. Choose the mollification scale so that the sum of these three errors relative to $\chi_n\psi$ is at most $1/n$. The resulting sequence $\psi_n\in C_c^\infty(\R^d)$ satisfies \eqref{eq:poisson-approximation}.
\end{proof}

\subsection{Equivalent test classes for the negative Sobolev norm}
\label{app:negative-norm}

\begin{proof}[Proof of \eqref{eq:Hminus-test-classes}]
Let $\sigma$ be an integrable signed density with $\int\sigma\dx=0$ and $\sigma/\pi_\lambda\in L^2(\pi_\lambda)$.
We first compare compactly supported smooth tests with $H^1(\pi_\lambda)$ tests.

For $v\in H^1(\pi_\lambda)$, choose smooth cutoffs $0\le\chi_n\le1$, equal to one on the ball of radius $n$, supported in the ball of radius $2n$, and satisfying $\norm{\nabla\chi_n}\le C/n$.
Then
\begin{align*}
 \Lp{\chi_n v-v}{2}&\longrightarrow0,\\
 \Lp{\nabla(\chi_n v)-\nabla v}{2}
 &\le
 \Lp{(\chi_n-1)\nabla v}{2}
 +\frac Cn\Lp{v}{2}
 \longrightarrow0.
\end{align*}
On each compact set, $\pi_\lambda$ is bounded above and below by positive constants.
Mollifying $\chi_n v$ and choosing the smoothing scales diagonally therefore gives $v_n\in C_c^\infty(\R^d)$ with
\[
 \Lp{v_n-v}{2}+\Lp{\nabla v_n-\nabla v}{2}
 \longrightarrow0.
\]
The pairing is continuous under this convergence, since
\[
 \left|\int(v_n-v)\sigma\dx\right|
 \le
 \Lp{\sigma/\pi_\lambda}{2}\Lp{v_n-v}{2}
 \longrightarrow0.
\]
If $\Lp{\nabla v}{2}\le1$, replace $v_n$ by
\[
 \frac{v_n}{\max\{1,\Lp{\nabla v_n}{2}\}}.
\]
This preserves convergence to $v$ in $H^1(\pi_\lambda)$ and makes every approximation admissible in \eqref{eq:Hminus-def}. Thus the supremum over $H^1(\pi_\lambda)$ equals the supremum over $C_c^\infty(\R^d)$.

It remains to compare with $C^1$ tests.
Let $v\in C^1(\R^d)$ satisfy $\Lp{\nabla v}{2}<\infty$. To verify that $v\in H^1(\pi_\lambda)$, define the bounded truncations
\[
 v_k:=\max\{-k,\min\{v,k\}\}.
\]
They belong to $H^1(\pi_\lambda)$ and satisfy $\norm{\nabla v_k}\le\norm{\nabla v}$ almost everywhere.
The Poincar\'e inequality gives
\[
 \int\left|v_k-\int v_k\dd\pi_\lambda\right|^2
                       \dd\pi_\lambda
 \le\frac1m\Lp{\nabla v}{2}^2.
\]
For $k>\sup_{\norm{x}\le1}|v(x)|$, one has $v_k=v$ on the unit ball. Restricting the preceding variance bound to that ball gives
\[
 \left|\int v_k\dd\pi_\lambda\right|
 \le
 \sup_{\norm{x}\le1}|v(x)|
 +\frac{\Lp{\nabla v}{2}}
 {\sqrt{m\,\pi_\lambda(\{x:\norm{x}\le1\})}}.
\]
The denominator is positive. Hence the means of $v_k$, and therefore their $L^2(\pi_\lambda)$ norms, are uniformly bounded. Since $v_k\to v$ pointwise, Fatou's lemma yields $v\in L^2(\pi_\lambda)$. Together with the assumed gradient bound, this proves $v\in H^1(\pi_\lambda)$.

Consequently, every admissible $C^1$ test is an admissible $H^1(\pi_\lambda)$ test. Since $C_c^\infty(\R^d)\subset C^1(\R^d)$ and the $C_c^\infty$ and $H^1$ suprema have already been shown equal, the $C^1$ supremum has the same value.
This proves \eqref{eq:Hminus-test-classes}.
\end{proof}

\section{Weak Regularity and Integrability Before Dissipation}\label[appendix]{app:regularity}
This appendix supplies the technical arguments used in
\cref{app:renyi}.
We first verify the properties of the auxiliary mollification.
We then prove the integrability statement
\cref{lem:prior-finiteness} independently of the dissipation
calculation.
Finally, \cref{app:approximation-limit} passes the quantitative
estimate to the original potential.
\subsection{Integration by parts and smooth approximation}\label{app:approximation}
We first justify the integrations by parts in
\cref{app:moments}.
Let $\Phi$ be a locally Lipschitz vector field such
that $\Phi$ and its weak divergence have at most
polynomial growth.
Choose $\chi_n\in C_c^\infty(\R^d)$ with
$0\le\chi_n\le1$, equal to one on the ball of radius
$n$, supported on the ball of radius $2n$, and
satisfying $\norm{\nabla\chi_n}_\infty\le C/n$.
Integration by parts with compact support gives
\[
 \int\chi_n\Div\Phi\dd\pi_\lambda
 =\int\chi_n\ip{\Phi}{\drift}\dd\pi_\lambda
  -\int\ip{\nabla\chi_n}{\Phi}\dd\pi_\lambda.
\]
Strong convexity gives Gaussian tails for
$\pi_\lambda$, and $\drift$ grows at most linearly.
The first two terms therefore converge by dominated
convergence, while
\[
 \left|
   \int\ip{\nabla\chi_n}{\Phi}\dd\pi_\lambda
 \right|
 \le\frac Cn\int\norm{\Phi}\dd\pi_\lambda
 \longrightarrow0.
\]
Thus
\[
 \int\Div\Phi\dd\pi_\lambda
 =\int\ip{\Phi}{\drift}\dd\pi_\lambda.
\]
The vector fields used in \cref{app:moments}
satisfy these conditions.

Consider the mollified potentials defined in
\cref{app:renyi-def}.
Convolution with the nonnegative kernel $\varrho_\delta$
preserves the structural bounds:
\[
 \begin{gathered}
 mI\preceq\nabla^2f^{(\delta)}\preceq L_fI,
 \qquad
 \tr\bigl(\nabla^2f^{(\delta)}\bigr)\le\tau_f,\\
 \norm{\nabla g_\lambda^{(\delta)}}\le G,
 \qquad
 0\preceq\nabla^2g_\lambda^{(\delta)}
       \preceq\lambda^{-1}I.
 \end{gathered}
\]
Thus the same parameters and step restrictions apply for
every $\delta>0$.
For each fixed $\delta$, the mollified potentials are smooth
and all their derivatives of order at least two are bounded.
These derivative bounds may depend on $\delta$.

The function $g_\lambda^{(\delta)}$ need not be the Moreau
envelope of the original $g$.
The estimates in \cref{app:renyi} use the convexity, gradient,
and Hessian bounds displayed above.
The convergence of the regularized potentials and their
associated laws is treated in \cref{app:approximation-limit}.

\subsection{Finiteness before the R\'enyi calculation}\label{app:finiteness}
\begin{proof}[Proof of \cref{lem:prior-finiteness}]
Fix $\alpha\ge2$, $\lambda$, and $h$ satisfying the restrictions
of \cref{lem:ratios}, and fix an auxiliary scale $\delta>0$.
Throughout this proof, the potentials, densities, kernels,
and maps are the regularized objects defined in
\cref{app:renyi-def}; their superscripts $(\delta)$ are omitted.
In particular, the function $f$ and its minimizer introduced
below belong to this regularized problem.
In this subsection only, $C$ and its subscripted versions may depend on the fixed problem data and on $\alpha,\lambda,h,\delta$. The step-size constant $c$ remains universal.

We first establish an exponential moment of the invariant law,
then derive density and gradient bounds uniform in
$t\in[0,h]$, and finally control the weighted score and joint
moments.
The finite bounds obtained in this proof may depend on the
fixed parameters and on $\delta$.
Their values are used only to justify the analytic operations
in \cref{app:renyi} and do not enter the quantitative
step restrictions.
For $hL_\lambda<c<1$, the map $T_h$ is $(1-mh)$-Lipschitz.
The Euler kernel therefore contracts $\mathcal P_2(\R^d)$ in $W_2$ and has a unique invariant law.

\paragraph{An exponential moment of the invariant law.}
We prove the exponential moment bound \eqref{eq:finite-exp} for $\inv$.
Let $x_f=\arg\min f$, $F=f-f(x_f)$, and $\eta=1-1/(8\alpha)$, only within this subsection.
Strong convexity and smoothness give
\begin{equation}\label{eq:finite-F-bounds}
 F(x)\ge\frac m2\norm{x-x_f}^2,
 \qquad
 2mF(x)\le\norm{\nabla F(x)}^2\le2L_fF(x).
\end{equation}
For
\[
 Y=x-h\drift(x)+\sqrt{2h}\,Z,
\]
the $L_f$-smoothness of $f$ gives
\[
 F(Y)-F(x)
 \le\ip{\nabla f(x)}{Y-x}
       +\frac{L_f}{2}\norm{Y-x}^2.
\]
Substituting $Y-x=-h\drift(x)+\sqrt{2h}\,Z$
and expanding the square yields
\begin{align*}
 F(Y)-F(x)
 \le{}&-h\ip{\nabla f(x)}{\drift(x)}
       +\frac{h^2L_f}{2}\norm{\drift(x)}^2\\
 &+\sqrt{2h}\,
       \ip{\nabla f(x)-hL_f\drift(x)}{Z}
       +hL_f\norm Z^2.
\end{align*}

For $\xi\in\R^d$ and $\theta<1/2$, completing
the square gives
\[
 -\frac12\norm z^2+\ip{\xi}{z}+\theta\norm z^2
 =
 -\frac{1-2\theta}{2}
       \left\|z-\frac{\xi}{1-2\theta}\right\|^2
 +\frac{\norm\xi^2}{2(1-2\theta)}.
\]
Integrating against the standard Gaussian density
therefore gives
\begin{equation}\label{eq:finite-gaussian}
 \E e^{\ip{\xi}{Z}+\theta\norm Z^2}
 =(1-2\theta)^{-d/2}
   \exp\left\{
      \frac{\norm\xi^2}{2(1-2\theta)}
   \right\}.
\end{equation}
The step-size restrictions ensure $2\eta hL_f<1$.
Apply \eqref{eq:finite-gaussian} with
\[
 \theta=\eta hL_f,
 \qquad
 \xi=\eta\sqrt{2h}
          \bigl(\nabla f(x)-hL_f\drift(x)\bigr).
\]
Since
\[
 \frac{Q_{\lambda,h}e^{\eta F}(x)}
      {e^{\eta F(x)}}
 =\E e^{\eta(F(Y)-F(x))},
\]
we obtain
\begin{align*}
 \log\frac{Q_{\lambda,h}e^{\eta F}(x)}
               {e^{\eta F(x)}}
 \le{}&-\frac d2\log(1-2\eta hL_f)
       -\eta h\ip{\nabla f(x)}{\drift(x)}\\
 &+\frac{\eta h^2L_f}{2}\norm{\drift(x)}^2
       +\frac{\eta^2h}{1-2\eta hL_f}
          \norm{\nabla f(x)-hL_f\drift(x)}^2\\
 ={}&-\frac d2\log(1-2\eta hL_f)\\
 &+\frac{\eta h}{1-2\eta hL_f}
   \left[
     \eta\norm{\nabla f(x)}^2
     -\ip{\nabla f(x)}{\drift(x)}
     +\frac{hL_f}{2}\norm{\drift(x)}^2
   \right].
\end{align*}
Finally, substitute $\drift=\nabla f+\nabla g_\lambda$ and use $\norm{\nabla g_\lambda}\le G$ to obtain
\begin{align}\label{eq:Lyap-exact}
 \log\frac{Q_{\lambda,h}e^{\eta F}(x)}{e^{\eta F(x)}}
 \le{}&-\frac d2\log(1-2\eta hL_f)\notag\\
 &+\frac{\eta h}{1-2\eta hL_f}
 \Bigl[(-(1-\eta)+hL_f/2)\norm{\nabla f(x)}^2\notag\\
 &\hspace{28mm}-(1-hL_f)\ip{\nabla f(x)}{\nabla g_\lambda(x)}
                    +(hL_f/2)G^2\Bigr].
\end{align}
The first condition in \eqref{eq:ratio-step}, together with $\tau_f\ge m$, gives $hL_f\le c/\alpha^2$.
Since $1-\eta=1/(8\alpha)$, taking the universal constant $c$ small enough ensures
\[
 hL_f\le\frac{1-\eta}{2},
 \qquad
 1-2\eta hL_f\ge\frac12.
\]
Hence
\[
 -(1-\eta)+\frac{hL_f}{2}
 \le-\frac34(1-\eta).
\]
For the mixed term, Young's inequality gives
\begin{align*}
 -(1-hL_f)
       \ip{\nabla f(x)}{\nabla g_\lambda(x)}
 \le G\norm{\nabla f(x)}
 \le\frac{1-\eta}{4}\norm{\nabla f(x)}^2
       +\frac{G^2}{1-\eta}.
\end{align*}
The bracket in \eqref{eq:Lyap-exact} is therefore at most
\[
 -\frac{1-\eta}{2}\norm{\nabla f(x)}^2
 +\left(\frac1{1-\eta}+\frac{hL_f}{2}\right)G^2.
\]
Also, $-\log(1-z)\le2z$ for $0\le z\le1/2$ gives
\[
 -\frac d2\log(1-2\eta hL_f)
 \le2\eta dhL_f.
\]
Using $\norm{\nabla f(x)}^2\ge2mF(x)$, we conclude that
\[
 \log\frac{Q_{\lambda,h}e^{\eta F}(x)}
               {e^{\eta F(x)}}
 \le-\eta(1-\eta)mhF(x)+C_\alpha h.
\]
Thus, with $c_\alpha=\eta(1-\eta)>0$,
\begin{equation}\label{eq:Lyap}
 Q_{\lambda,h}e^{\eta F}
 \le e^{\eta F}\exp\{-c_\alpha mhF+C_\alpha h\}.
\end{equation}
Choose a finite $K>0$ such that $c_\alpha mK\ge C_\alpha+1$.
For $F(x)\ge K$, \eqref{eq:Lyap} gives
\[
 Q_{\lambda,h}e^{\eta F}(x)
 \le e^{-h}e^{\eta F(x)}.
\]
For $F(x)<K$, it gives
\[
 Q_{\lambda,h}e^{\eta F}(x)
 \le e^{\eta K+C_\alpha h}.
\]
Combining the two bounds, we obtain, for every $x$,
\[
 Q_{\lambda,h}e^{\eta F}(x)
 \le e^{-h}e^{\eta F(x)}
       +e^{\eta K+C_\alpha h}.
\]

Start the Euler chain at $X_0=x_f$. Since $F(x_f)=0$, taking expectations and iterating the preceding inequality gives
\[
 \E e^{\eta F(X_n)}
 \le e^{-nh}
       +e^{\eta K+C_\alpha h}
          \sum_{j=0}^{n-1}e^{-jh}
 \le1+\frac{e^{\eta K+C_\alpha h}}{1-e^{-h}}.
\]
Each expectation is finite by induction, and the bound is independent of $n$.
The $W_2$ contraction gives convergence of the law of $X_n$ to $\inv$.
Since $e^{\eta F}$ is continuous and nonnegative, lower semicontinuity yields
\begin{equation}\label{eq:finite-exp}
 \int e^{\eta F}\inv\dx
 \le\liminf_{n\to\infty}\E e^{\eta F(X_n)}
 <\infty.
\end{equation}

\paragraph{Bounds for the density and its gradient.}
We bound $\nu_t$ and its gradient uniformly for $0\le t\le h$, as stated in \eqref{eq:finite-envelope}.

We first prove the comparison \eqref{eq:finite-weight} between $F(y)$ and $F(z)$.
For $0<\beta<\gamma$, Taylor's upper bound gives
\begin{align*}
 \beta F(y)-\gamma F(z)
 \le{}&-(\gamma-\beta)F(z)
       +\beta\ip{\nabla F(z)}{y-z}
       +\frac{\beta L_f}{2}\norm{y-z}^2.
\end{align*}
By Young's inequality,
\begin{align*}
 \beta\ip{\nabla F(z)}{y-z}
 &\le
 \frac{\gamma-\beta}{2L_f}\norm{\nabla F(z)}^2
 +\frac{\beta^2L_f}{2(\gamma-\beta)}\norm{y-z}^2\\
 &\le
 (\gamma-\beta)F(z)
 +\frac{\beta^2L_f}{2(\gamma-\beta)}\norm{y-z}^2,
\end{align*}
where the second inequality uses $\norm{\nabla F(z)}^2\le2L_fF(z)$.
Substitution gives
\begin{equation}\label{eq:finite-weight}
 \beta F(y)-\gamma F(z)
 \le\frac{\beta\gamma L_f}{2(\gamma-\beta)}
          \norm{y-z}^2.
\end{equation}

We also need an upper bound for $F(T_tx)$.
For $0\le t\le h$, Taylor's upper bound gives
\begin{align*}
 F(T_tx)-F(x)
 &\le
 -t\ip{\nabla f(x)}{\drift(x)}
 +\frac{t^2L_f}{2}\norm{\drift(x)}^2\\
 &=
 -\frac t2\norm{\nabla f(x)}^2
 +\frac t2\norm{\nabla g_\lambda(x)}^2
 -\frac{t(1-tL_f)}2\norm{\drift(x)}^2\\
 &\le
 -\frac t2\norm{\nabla f(x)}^2+\frac t2G^2
 \le\frac t2G^2,
\end{align*}
using $tL_f\le1$ and $\norm{\nabla g_\lambda}\le G$.
Thus
\begin{equation}\label{eq:finite-descent}
 F(T_tx)\le F(x)+\frac t2G^2,
 \qquad 0\le t\le h.
\end{equation}
Choose the two intermediate exponents
\[
 \beta_1:=1-\frac1{4\alpha},\qquad
 \beta_2:=1-\frac1{2\alpha},
\]
and decrease the universal step constant so that $h\alpha L_f\le1/32$.

The stationary identity and \eqref{eq:heat-def} give
\[
 \inv(y)
 =(4\pi h)^{-d/2}
   \int e^{-\norm{y-T_hx}^2/(4h)}\inv(x)\dx.
\]
Differentiating the Gaussian kernel gives
\[
 \nabla\inv(y)
 =-(4\pi h)^{-d/2}
   \int
      \frac{y-T_hx}{2h}
      e^{-\norm{y-T_hx}^2/(4h)}
      \inv(x)\dx.
\]
This differentiation is justified because the kernel and its spatial derivatives are bounded for fixed $h>0$.
Since
\[
 (4\pi h)^{-d/2}
 \left(1+\frac{\norm{y-T_hx}}{2h}\right)
 e^{-\norm{y-T_hx}^2/(4h)}
 \le C_h e^{-\norm{y-T_hx}^2/(8h)},
\]
the density and its gradient can be estimated using the same Gaussian bound.
In \eqref{eq:finite-weight}, the choice $(\gamma,\beta)=(\eta,\beta_1)$ gives a coefficient at most $4\alpha L_f$.
Thus
\begin{align*}
 &e^{\beta_1F(y)}\bigl(\inv(y)+\norm{\nabla\inv(y)}\bigr)\\
 &\quad\le C_h\int e^{\eta F(T_hx)}
 \exp\!\left\{-\left(\frac1{8h}-4\alpha L_f\right)
                     \norm{y-T_hx}^2\right\}\inv(x)\dx\\
 &\quad\le C_he^{\eta hG^2/2}\int e^{\eta F(x)}\inv(x)\dx<\infty.
\end{align*}
The last inequality uses $h\alpha L_f\le1/32$, \eqref{eq:finite-descent} with $t=h$, and \eqref{eq:finite-exp}.
The bound is independent of $y$.

The change-of-variables formula gives
\[
 ((T_t)_\#\inv)(y)
 =\inv(T_t^{-1}y)\,|\det DT_t^{-1}(y)|.
\]
By \cref{prop:pushforward},
\[
 DT_t^{-1}(y)
 =
 \bigl(
   I-t\nabla^2U_\lambda(T_t^{-1}y)
 \bigr)^{-1},
 \qquad
 \norm{DT_t^{-1}(y)}_{\mathrm{op}}
 \le(1-hL_\lambda)^{-1}.
\]
The determinant is positive.
Differentiating the density formula gives
\begin{align*}
 \nabla((T_t)_\#\inv)(y)
 ={}&
 \det DT_t^{-1}(y)\,
 [DT_t^{-1}(y)]^{\mathsf T}
 \nabla\inv(T_t^{-1}y)\\
 &+\inv(T_t^{-1}y)\,
       \nabla\det DT_t^{-1}(y).
\end{align*}
At the fixed smoothing scale, the third derivatives
of $U_\lambda$ are bounded.
Differentiating the inverse-matrix formula above
therefore shows that
$\det DT_t^{-1}$ and
$\nabla\det DT_t^{-1}$
are uniformly bounded for $0\le t\le h$.

The preceding bounds on $\inv$ and $\nabla\inv$
now imply
\[
 ((T_t)_\#\inv)(y)
 +\norm{\nabla((T_t)_\#\inv)(y)}
 \le C e^{-\beta_1F(T_t^{-1}y)}.
\]
Applying \eqref{eq:finite-descent} with $x=T_t^{-1}y$ gives
\[
 F(y)\le F(T_t^{-1}y)+tG^2/2.
\]
Consequently,
\begin{equation}\label{eq:finite-push-envelope}
 ((T_t)_\#\inv)(y)
 +\norm{\nabla((T_t)_\#\inv)(y)}
 \le C e^{-\beta_1F(y)},
 \qquad 0\le t\le h.
\end{equation}
with one constant for the whole time interval.

For the final heat convolution, apply \eqref{eq:finite-weight} with $(\gamma,\beta)=(\beta_1,\beta_2)$, whose coefficient is at most $2\alpha L_f$.
By \eqref{eq:finite-push-envelope} and \eqref{eq:finite-F-bounds}, both the pushed density and its gradient are integrable. Spatial integration by parts therefore gives
\[
 \nabla\heat_t((T_t)_\#\inv)
 =
 \heat_t\nabla((T_t)_\#\inv).
\]
For $t>0$ it follows that
\begin{align*}
 &e^{\beta_2F(y)}\bigl(\nu_t(y)+\norm{\nabla\nu_t(y)}\bigr)\\
 &\quad\le C(4\pi t)^{-d/2}
 \int\exp\!\left\{-\left(\frac1{4t}-2\alpha L_f\right)
                         \norm{y-z}^2\right\}\dd z\\
 &\quad=C(1-8\alpha tL_f)^{-d/2}\le C'.
\end{align*}
The case $t=0$ follows from \eqref{eq:finite-push-envelope}, since $T_0$ is the identity, $F\ge0$, and $\beta_1>\beta_2$.
Consequently, with a constant uniform for $0\le t\le h$,
\begin{equation}\label{eq:finite-envelope}
 \nu_t(y)+\norm{\nabla\nu_t(y)}
 \le C_{\alpha,h}\exp\{-(1-1/(2\alpha))F(y)\}.
\end{equation}
\paragraph{Regularity and endpoint continuity.}

We show that the first time derivative and the
spatial derivatives up to order two of $\nu_t$
exist and are continuous for $0<t<h$.
We then check continuity at the two time endpoints.

For $t>0$, write the Gaussian representation as
\[
 \nu_t(y)=\int\kappa_t(y,x)\inv(x)\dx,
 \qquad
 \kappa_t(y,x)
 =(4\pi t)^{-d/2}
   \exp\left\{-\frac{\norm{y-T_tx}^2}{4t}\right\}.
\]
Here $T_tx=x-t\drift(x)$, while $\inv$ is fixed
throughout the interpolation.

The spatial derivatives of the kernel are
\begin{align*}
 \nabla_y\kappa_t(y,x)
 &=-\frac{y-T_tx}{2t}\,\kappa_t(y,x),\\
 \nabla_y^2\kappa_t(y,x)
 &=
 \left[
   \frac{(y-T_tx)(y-T_tx)^{\mathsf T}}{4t^2}
   -\frac{I_d}{2t}
 \right]\kappa_t(y,x).
\end{align*}
For the time derivative, note that $\partial_t(y-T_tx)=\drift(x)$.
Differentiating both the prefactor and the exponent gives
\[
 \partial_t\kappa_t(y,x)
 =
 \left[
   -\frac d{2t}
   +\frac{\norm{y-T_tx}^2}{4t^2}
   -\frac{\ip{y-T_tx}{\drift(x)}}{2t}
 \right]\kappa_t(y,x).
\]

Fix $0<\varepsilon<h$.
The functions
\[
 z\longmapsto
 (1+\norm z+\norm z^2)e^{-\norm z^2/4}
\]
are bounded on $\R^d$.
Applying this observation with $z=(y-T_tx)/\sqrt t$ to the displayed derivatives gives, for $\varepsilon\le t\le h$ and all $x,y$,
\begin{align*}
 &|\kappa_t(y,x)|
   +\norm{\nabla_y\kappa_t(y,x)}
   +\norm{\nabla_y^2\kappa_t(y,x)}
   +|\partial_t\kappa_t(y,x)|\\
 &\qquad\le
 C_\varepsilon\bigl(1+\norm{\drift(x)}\bigr).
\end{align*}
This bound is independent of $t$ and $y$ in the stated range.

The drift grows at most linearly, and the quadratic lower bound in \eqref{eq:finite-F-bounds} implies
\[
 1+\norm{\drift(x)}
 \le C(1+\norm{x-x_f})
 \le C'e^{\eta F(x)}.
\]
Thus \eqref{eq:finite-exp} makes the common bound integrable against $\inv(x)\dx$.

We may therefore differentiate the integral once in time and twice in space:
\begin{align*}
 \partial_t\nu_t(y)
 &=\int\partial_t\kappa_t(y,x)\inv(x)\dx,\\
 \nabla_y^j\nu_t(y)
 &=\int\nabla_y^j\kappa_t(y,x)\inv(x)\dx,
 \qquad j=1,2.
\end{align*}
For each fixed $x$, the kernel and all the displayed derivatives are continuous in $(t,y)$.
Dominated convergence, using the same common bound, shows that their integrals are also continuous in $(t,y)$.
Since $\varepsilon>0$ is arbitrary, $\nu$ is locally $C^{1,2}$ on $(0,h)\times\R^d$.

Moreover, $\kappa_t(y,x)>0$ for every $t>0$ and every $x,y$.
Since $\inv$ is a probability density, this gives $\nu_t(y)>0$.
The same dominated-convergence argument applies as $t\uparrow h$, so
\[
 \nu_t(y)\longrightarrow\nu_h(y)=\inv(y).
\]

We now prove continuity at $t=0$.
First, the inverse-map bound in \cref{prop:pushforward} gives
\begin{align*}
 \norm{T_t^{-1}y-y}
 &=\norm{T_t^{-1}y-T_t^{-1}(T_ty)}\\
 &\le\frac{1}{1-tL_\lambda}\norm{y-T_ty}\\
 &=\frac{t}{1-tL_\lambda}\norm{\drift(y)}
 \longrightarrow0.
\end{align*}
The inverse derivative also satisfies
\[
 DT_t^{-1}(y)
 =
 \bigl(I-t\nabla^2U_\lambda(T_t^{-1}y)\bigr)^{-1},
 \qquad
 \norm{DT_t^{-1}(y)-I_d}_{\mathrm{op}}
 \le\frac{tL_\lambda}{1-tL_\lambda}
 \longrightarrow0.
\]
Using the change-of-variables formula \eqref{eq:push-density} and the continuity of $\inv=\nu_h$, we obtain
\[
 ((T_t)_\#\inv)(y)
 =
 \inv(T_t^{-1}y)\det DT_t^{-1}(y)
 \longrightarrow\inv(y).
\]

We next estimate the effect of the Gaussian increment.
Equation \eqref{eq:finite-push-envelope} and $F\ge0$ give
\[
 \sup_{0\le r\le h}
 \norm{\nabla((T_r)_\#\inv)}_\infty\le C.
\]
Thus these densities are Lipschitz with the same constant for all $0\le r\le h$.
By \eqref{eq:heat-def} and $\nu_t=\heat_t((T_t)_\#\inv)$,
\begin{align*}
 &|\nu_t(y)-((T_t)_\#\inv)(y)|\\
 &\quad=
 \left|
 \E\left[
   ((T_t)_\#\inv)(y+\sqrt{2t}\,Z)
   -((T_t)_\#\inv)(y)
 \right]
 \right|\\
 &\quad\le C\sqrt{2t}\,\E\norm Z
 \longrightarrow0.
\end{align*}
Together with
$((T_t)_\#\inv)(y)\to\inv(y)$, this proves
\[
 \nu_t(y)\longrightarrow\inv(y)=\nu_0(y)
 \qquad\text{as }t\downarrow0.
\]

\paragraph{Density-ratio and score integrals.}

Since $g_\lambda$ is $G$-Lipschitz,
\[
 \pi_\lambda(y)
 \ge C^{-1}
       \exp\{-F(y)-G\norm{y-x_f}\}.
\]
Combining this with \eqref{eq:finite-envelope} gives
\begin{align*}
 u(y)^\alpha\pi_\lambda(y)
 &=\nu_t(y)^\alpha\pi_\lambda(y)^{1-\alpha}\\
 &\le C\exp\left\{
   -\alpha\left(1-\frac1{2\alpha}\right)F(y)
   +(\alpha-1)\bigl(F(y)+G\norm{y-x_f}\bigr)
 \right\}\\
 &=C\exp\left\{
      -\frac12F(y)+(\alpha-1)G\norm{y-x_f}
    \right\}.
\end{align*}
Similarly,
\[
 u(y)^{3\alpha/2}\pi_\lambda(y)
 \le C\exp\left\{
      -\frac14F(y)
      +\left(\frac{3\alpha}{2}-1\right)
             G\norm{y-x_f}
    \right\}.
\]
Both bounds are integrable because
\[
 F(y)\ge\frac m2\norm{y-x_f}^2.
\]
Their constants are uniform for $0\le t\le h$.

For the score, use $s=\nabla\log\nu_t+\drift$ to obtain
\begin{align*}
 u^\alpha\pi_\lambda\norm s^2
 \le{}
 2\pi_\lambda^{1-\alpha}\nu_t^{\alpha-2}
       \norm{\nabla\nu_t}^2
 +2\pi_\lambda^{1-\alpha}\nu_t^\alpha
       \norm{\drift}^2.
\end{align*}
Since $\alpha\ge2$, the factor $\nu_t^{\alpha-2}$ can be bounded using \eqref{eq:finite-envelope}.
The same equation bounds $\nabla\nu_t$, and $\drift$ grows at most linearly.
Thus the sum of the two terms is bounded by
\[
 C\bigl(1+\norm{y-x_f}^2\bigr)
 \exp\left\{
      -\frac12F(y)+(\alpha-1)G\norm{y-x_f}
     \right\}.
\]
This bound is integrable and uniform in time.
We have therefore proved
\[
 \sup_{0\le t\le h}
 \int\bigl[
      u^\alpha(1+\norm s^2)+u^{3\alpha/2}
     \bigr]\dd\pi_\lambda<\infty.
\]

The pointwise endpoint continuity of $\nu_t$ and the integrable bound for $u^\alpha\pi_\lambda$ give the following limits by dominated convergence:
\begin{equation}\label{eq:finite-moment-continuity}
 \begin{aligned}
 \lim_{t\downarrow0}
       \int u(t)^\alpha\dd\pi_\lambda
 &=
 \int\left(\frac{\inv}{\pi_\lambda}\right)^\alpha
       \dd\pi_\lambda,\\
 \lim_{t\uparrow h}
       \int u(t)^\alpha\dd\pi_\lambda
 &=
 \int\left(\frac{\inv}{\pi_\lambda}\right)^\alpha
       \dd\pi_\lambda.
 \end{aligned}
\end{equation}
Since $\nu_0=\nu_h=\inv$, this proves continuity of $t\mapsto\int u(t)^\alpha\dd\pi_\lambda$ at both endpoints.

\paragraph{Joint weighted moments.}

We prove the joint moment bound in \cref{lem:prior-finiteness}.
The upper bound for $\nu_t$ and the lower bound for $\pi_\lambda$ give
\[
 u(y)^{\alpha-1}
 \le C\exp\left\{
       \frac{\alpha-1}{2\alpha}F(y)
       +(\alpha-1)G\norm{y-x_f}
     \right\}.
\]
Strong convexity and Young's inequality give
\[
 (\alpha-1)G\norm{y-x_f}
 \le\frac18F(y)
       +\frac{4(\alpha-1)^2G^2}{m}.
\]
Since $(\alpha-1)/(2\alpha)\le1/2$, it follows that
\[
 u(y)^{\alpha-1}\le C e^{5F(y)/8}.
\]

For every integer $k\ge0$, the quadratic lower bound in \eqref{eq:finite-F-bounds} gives
\[
 (1+\norm Y)^k\le C_k e^{F(Y)/32}.
\]
Consequently,
\begin{align*}
 &u(Y)^{\alpha-1}(1+\norm X+\norm Y)^k\\
 &\qquad\le
 C_k(1+\norm X)^k
       e^{(5/8+1/32)F(Y)}\\
 &\qquad=
 C_k(1+\norm X)^k e^{21F(Y)/32}.
\end{align*}

Conditional on $X$, Taylor's upper bound gives
\[
 F(Y)
 \le F(T_tX)
       +\sqrt{2t}\ip{\nabla F(T_tX)}{Z}
       +tL_f\norm Z^2.
\]
Applying \eqref{eq:finite-gaussian} yields
\begin{align*}
 \E\!\left[e^{21F(Y)/32}\mid X\right]
 \le{}&
 \left(1-\frac{21}{16}tL_f\right)^{-d/2}
 \exp\left\{
   \frac{21}{32}F(T_tX)
   +\frac{(21/32)^2t}
          {1-\frac{21}{16}tL_f}
       \norm{\nabla F(T_tX)}^2
 \right\}\\
 \le{}&
 \left(1-\frac{21}{16}tL_f\right)^{-d/2}
 \exp\left\{
   \frac{21F(T_tX)}
        {32(1-\frac{21}{16}tL_f)}
 \right\},
\end{align*}
where the second inequality uses
$\norm{\nabla F}^2\le2L_fF$.

The condition $h\alpha L_f\le1/32$ and $\alpha\ge2$ imply $tL_f\le1/64$.
In particular,
\[
 \frac{21}{32(1-\frac{21}{16}tL_f)}
 \le\frac34,
\]
and the Gaussian prefactor is bounded uniformly for $0\le t\le h$.
Using \eqref{eq:finite-descent} with $x=X$, we obtain
\[
 \E\!\left[e^{21F(Y)/32}\mid X\right]
 \le C e^{3F(X)/4},
 \qquad 0\le t\le h.
\]

Finally, $\eta=1-1/(8\alpha)\ge15/16$ gives $\eta-3/4\ge3/16$.
The quadratic lower bound in \eqref{eq:finite-F-bounds} therefore gives
\[
 (1+\norm X)^k e^{3F(X)/4}
 \le C_k e^{\eta F(X)}.
\]
Combining the preceding estimates gives
\begin{align*}
 &\E\!\left[
      u(Y)^{\alpha-1}(1+\norm X+\norm Y)^k
     \right]\\
 &\qquad\le
 C_k\E\!\left[
      (1+\norm X)^k e^{3F(X)/4}
     \right]\\
 &\qquad\le
 C_k'\int e^{\eta F(x)}\inv(x)\dx
 <\infty.
\end{align*}
The last inequality is \eqref{eq:finite-exp}, and the bound is uniform for $0\le t\le h$.

It remains to bound the integral containing $e_t$.
By its definition,
\[
 e_t(Y)
 =\E[\drift(Y)-\drift(X)\mid Y].
\]
Conditional Jensen's inequality gives
\begin{align*}
 \int u^\alpha\norm{e_t}^2\dd\pi_\lambda
 &=\E\!\left[
       u(Y)^{\alpha-1}\norm{e_t(Y)}^2
      \right]\\
 &\le\E\!\left[
       u(Y)^{\alpha-1}
       \norm{\drift(Y)-\drift(X)}^2
      \right]\\
 &\le L_\lambda^2
       \E\!\left[
        u(Y)^{\alpha-1}\norm{Y-X}^2
       \right].
\end{align*}
The last expression is uniformly finite by the joint moment bound with $k=2$.

This completes the proof of \cref{lem:prior-finiteness}.
\end{proof}

\subsection{Removing the auxiliary mollification}
\label{app:approximation-limit}

We now restore the superscripts $(\delta)$ and let $\delta\downarrow0$, with $\alpha,\lambda$, and $h$ fixed.
The global Lipschitz bound on $\drift$ gives
\[
 \varepsilon_\delta
 :=\norm{\nabla U_\lambda^{(\delta)}-\drift}_\infty
 \longrightarrow0.
\]
By the definitions of $f^{(\delta)}$ and $g_\lambda^{(\delta)}$,
\[
 U_\lambda^{(\delta)}(x)
 =(U_\lambda*\varrho_\delta)(x)
 =\int U_\lambda(x-\delta z)\varrho(z)\dd z.
\]
The symmetry of $\varrho$ gives
\[
 \int z\varrho(z)\dd z=0.
\]
Since $\nabla U_\lambda$ is $L_\lambda$-Lipschitz, Taylor's remainder satisfies
\[
 \left|
 U_\lambda(x-\delta z)-U_\lambda(x)
 +\delta\ip{\nabla U_\lambda(x)}{z}
 \right|
 \le\frac{L_\lambda\delta^2}{2}\norm z^2.
\]
Using the zero first moment, we can write
\begin{align*}
 U_\lambda^{(\delta)}(x)-U_\lambda(x)
 =\int\Bigl[
 U_\lambda(x-\delta z)-U_\lambda(x)
 +\delta\ip{\nabla U_\lambda(x)}{z}
 \Bigr]\varrho(z)\dd z.
\end{align*}
Integrating the remainder bound and taking the supremum over $x$ therefore gives
\[
 \norm{U_\lambda^{(\delta)}-U_\lambda}_\infty
 \le\frac{L_\lambda\delta^2}{2}
       \int\norm z^2\varrho(z)\dd z
 =O(L_\lambda\delta^2).
\]
The implicit constant depends only on the fixed kernel $\varrho$.
Consequently, $\pi_\lambda^{(\delta)}/\pi_\lambda$ converges uniformly to one.
In particular, $\pi_\lambda^{(\delta)}\le2\pi_\lambda$ for all sufficiently small $\delta$. Since $\pi_\lambda$ has a finite second moment, dominated convergence gives both weak convergence and convergence of the second moments. By the standard characterization of $W_2$ convergence,
\[
 W_2(\pi_\lambda^{(\delta)},\pi_\lambda)
 \longrightarrow0.
\]

Recall that
\[
 \varepsilon_\delta
 :=\norm{\nabla U_\lambda^{(\delta)}-\drift}_\infty,
 \qquad
 T_h^{(\delta)}x
 :=x-h\nabla U_\lambda^{(\delta)}(x).
\]
The smoothed and original potentials satisfy the same lower and upper Hessian bounds. Hence the contraction argument in \cref{prop:invariant} gives
\[
 \norm{T_h^{(\delta)}x-T_h^{(\delta)}y}
 \le(1-mh)\norm{x-y}.
\]
Also,
\[
 \norm{T_h^{(\delta)}x-T_hx}
 \le h\varepsilon_\delta.
\]

Choose an optimal coupling $(X_\delta,X)$ of $\widehat\pi_{\lambda,h}^{(\delta)}$ and $\inv$, so that
\[
 \left(\E\norm{X_\delta-X}^2\right)^{1/2}
 =
 W_2\bigl(
      \widehat\pi_{\lambda,h}^{(\delta)},\inv
     \bigr).
\]
Let $Z\sim N(0,I_d)$ be independent of this pair.
By invariance,
\[
 T_h^{(\delta)}X_\delta+\sqrt{2h}\,Z
 \sim\widehat\pi_{\lambda,h}^{(\delta)},
 \qquad
 T_hX+\sqrt{2h}\,Z\sim\inv.
\]
These two updated random variables therefore give another coupling of the same invariant laws.
Their Gaussian increments cancel in the difference. The definition of $W_2$ and the $L^2$ triangle inequality give
\begin{align*}
 W_2\bigl(
      \widehat\pi_{\lambda,h}^{(\delta)},\inv
     \bigr)
 &\quad\le
 \left(
   \E\norm{T_h^{(\delta)}X_\delta-T_hX}^2
 \right)^{1/2}\\
 &\quad\le
 \left(
   \E\norm{
       T_h^{(\delta)}X_\delta-T_h^{(\delta)}X
      }^2
 \right)^{1/2}
 +
 \left(
   \E\norm{T_h^{(\delta)}X-T_hX}^2
 \right)^{1/2}\\
 &\quad\le
 (1-mh)
 \left(\E\norm{X_\delta-X}^2\right)^{1/2}
 +h\varepsilon_\delta\\
 &\quad=
 (1-mh)
 W_2\bigl(
      \widehat\pi_{\lambda,h}^{(\delta)},\inv
     \bigr)
 +h\varepsilon_\delta.
\end{align*}
Rearranging and using $h>0$ gives
\[
 W_2\bigl(
      \widehat\pi_{\lambda,h}^{(\delta)},\inv
     \bigr)
 \le\frac{\varepsilon_\delta}{m}
 \longrightarrow0
 \qquad\text{as }\delta\downarrow0.
\]
The same coupling for the interpolations, together with $\operatorname{Lip}(T_t^{(\delta)})\le1-mt$, yields
\[
 \begin{aligned}
 W_2(\nu_t^{(\delta)},\nu_t)
 &\le
 (1-mt)
 W_2\bigl(\widehat\pi_{\lambda,h}^{(\delta)},
          \widehat\pi_{\lambda,h}\bigr)
 +t\varepsilon_\delta\\
 &\le\frac{\varepsilon_\delta}{m},
 \qquad 0\le t\le h.
 \end{aligned}
\]
Thus the regularized interpolations converge to the original interpolation uniformly in time in $W_2$.

Under the restrictions of \cref{lem:ratios}, the argument in \cref{app:renyi} gives
\[
 \mathcal R_\alpha
   (\nu_t^{(\delta)}\Vert\pi_\lambda^{(\delta)})
 \le\frac14
 \qquad
 \text{for every }\delta>0
 \text{ and }0\le t\le h.
\]
For each fixed $t\in[0,h]$, the $W_2$ convergence established above implies the weak convergence
\[
 \nu_t^{(\delta)}\Rightarrow\nu_t,
 \qquad
 \pi_\lambda^{(\delta)}\Rightarrow\pi_\lambda.
\]
Joint lower semicontinuity of the R\'enyi divergence, together with the uniform bound for the regularized problem, therefore gives
\[
 \mathcal R_\alpha(\nu_t\Vert\pi_\lambda)
 \le
 \liminf_{\delta\downarrow0}
 \mathcal R_\alpha
   (\nu_t^{(\delta)}\Vert\pi_\lambda^{(\delta)})
 \le\frac14.
\]
The same bound holds for every $t\in[0,h]$, so taking the supremum in time and using \eqref{eq:renyi-def} proves \eqref{eq:ratios}.
The same argument also preserves the full quantitative bound \eqref{eq:Renyi-final}.

\phantomsection
\addcontentsline{toc}{section}{References}
\bibliographystyle{plainnat}
\bibliography{reference}

\begin{thebibliography}{9}
\providecommand{\natexlab}[1]{#1}
\providecommand{\url}[1]{\texttt{#1}}
\expandafter\ifx\csname urlstyle\endcsname\relax
  \providecommand{\doi}[1]{doi: #1}\else
  \providecommand{\doi}{doi: \begingroup \urlstyle{rm}\Url}\fi

\bibitem[Chewi(2026)]{Chewi}
Sinho Chewi.
\newblock {Log-Concave Sampling}.
\newblock Book manuscript, 2026.
\newblock URL \url{https://chewisinho.github.io/main.pdf}.
\newblock Final draft, August 20, 2026.

\bibitem[Chewi et~al.(2024)Chewi, Erdogdu, Li, Shen, and Zhang]{ChewiEtAl}
Sinho Chewi, Murat~A. Erdogdu, Mufan~Bill Li, Ruoqi Shen, and Matthew Zhang.
\newblock Analysis of {Langevin Monte Carlo} from {Poincar\'e} to {Log-Sobolev}.
\newblock arXiv preprint arXiv:2112.12662v2, 2024.
\newblock URL \url{https://arxiv.org/abs/2112.12662v2}.

\bibitem[Durmus et~al.(2018)Durmus, Moulines, and Pereyra]{DMP2018}
Alain Durmus, {\'E}ric Moulines, and Marcelo Pereyra.
\newblock Efficient {Bayesian} computation by proximal {Markov} chain {Monte Carlo}: When {Langevin} meets {Moreau}.
\newblock \emph{SIAM Journal on Imaging Sciences}, 11\penalty0 (1):\penalty0 473--506, 2018.
\newblock \doi{10.1137/16M1108340}.

\bibitem[Fan et~al.(2023)Fan, Yuan, and Chen]{FanYuanChen2023}
Jiaojiao Fan, Bo~Yuan, and Yongxin Chen.
\newblock Improved dimension dependence of a proximal algorithm for sampling.
\newblock In \emph{Proceedings of Thirty Sixth Conference on Learning Theory}, volume 195 of \emph{Proceedings of Machine Learning Research}, pages 1473--1521. PMLR, 2023.
\newblock URL \url{https://proceedings.mlr.press/v195/fan23a.html}.

\bibitem[Pedrotti and Whalley(2026)]{PedrottiWhalley2026}
Francesco Pedrotti and Peter~A. Whalley.
\newblock {Wasserstein} mixing time of the unadjusted {Langevin} algorithm.
\newblock arXiv preprint arXiv:2608.02430, 2026.

\bibitem[Peyre(2016)]{Peyre}
R{\'e}mi Peyre.
\newblock Comparison between {$W_2$} distance and {$\dot{H}^{-1}$} norm, and localisation of {Wasserstein} distance.
\newblock arXiv preprint arXiv:1104.4631v2, 2016.
\newblock URL \url{https://arxiv.org/abs/1104.4631v2}.

\bibitem[Simon(2018)]{SimonGMT}
Leon Simon.
\newblock Introduction to geometric measure theory.
\newblock Lecture notes, 2018.
\newblock URL \url{https://math.stanford.edu/~lms/ntu-gmt-text.pdf}.
\newblock {NTU} lectures, revised March 7, 2018.

\bibitem[Xin and Zhang(2026{\natexlab{a}})]{ActiveTrace}
Yuchen Xin and Zhihua Zhang.
\newblock Active-trace complexity bounds for {Moreau--Yosida} unadjusted {Langevin} sampling.
\newblock arXiv preprint arXiv:2608.13467v1, 2026{\natexlab{a}}.
\newblock URL \url{https://arxiv.org/abs/2608.13467v1}.

\bibitem[Xin and Zhang(2026{\natexlab{b}})]{XinZhang}
Yuchen Xin and Zhihua Zhang.
\newblock Poisson-corrector complexity bounds for {Moreau--Yosida} unadjusted {Langevin} sampling.
\newblock arXiv preprint arXiv:2609.12594v1, 2026{\natexlab{b}}.
\newblock URL \url{https://arxiv.org/abs/2609.12594v1}.

\end{thebibliography}

\end{document}